\documentclass{article}

\usepackage{fullpage}

\usepackage{url}            
\usepackage{booktabs}       % professional-quality tables\
\usepackage{nicefrac}       % compact symbols for 1/2, etc.
\usepackage{microtype}      % microtypography
\usepackage{multirow}

\usepackage[T1]{fontenc}
\usepackage{amsmath,amssymb,amsfonts,amsthm,textcomp,graphicx,nicefrac,mathtools,mathrsfs, dsfont} 
\usepackage{tikz}

\usepackage{enumitem}
\usepackage{bbm}
\usepackage[utf8]{inputenc}
\usepackage{enumitem}

\usepackage{algorithm}
\usepackage[noend]{algpseudocode}

\makeatletter
\def\BState{\State\hskip-\ALG@thistlm}
\makeatother

\def\<{{\langle}}
\def\>{{\rangle}}

\renewcommand{\exp}[1]{\operatorname{exp}\left(#1\right)} % Exponential
\providecommand{\argmax}{\mathop\mathrm{arg\, max}} % Defining math symbols

\newcommand{\indi}{\mathds{1}}

\def\P{\mathbb{P}} % Probability symbol

\newcommand{\s}{\mathcal{S}}%calligraphic S
\newcommand{\sss}{\mathscr{S}}%script S
\newcommand{\rr}{\mathcal{R}}
\newcommand{\tcal}{\mathcal{T}}
\newcommand{\cc}{\mathcal{C}}

\usepackage{graphicx}
\newcommand\smallo{
  \mathchoice
    {{\scriptstyle\mathcal{O}}}% \displaystyle
    {{\scriptstyle\mathcal{O}}}% \textstyle
    {{\scriptscriptstyle\mathcal{O}}}% \scriptstyle
    {\scalebox{.6}{$\scriptscriptstyle\mathcal{O}$}}%\scriptscriptstyle
  }

\newtheoremstyle{dotless}{}{}{\itshape}{}{\bfseries}{}{ }{}
\theoremstyle{dotless}

\newtheorem*{defi}{Definition}
\theoremstyle{plain}

\newtheorem{myth}{Theorem}

\newtheorem{myprop}[myth]{Proposition}
\newtheorem{mylem}[myth]{Lemma}

\newtheoremstyle{named}{}{}{\itshape}{}{\bfseries}{.}{.5em}{#1 #3}
\theoremstyle{named}
\newtheorem*{namthm*}{Theorem}

\usepackage{chngcntr}

\usepackage[colorlinks=true, citecolor = blue]{hyperref}
\usepackage{cleveref}
\crefname{myth}{Theorem}{Theorems} 

\newcounter{parentnumber}

\usepackage[
backend=biber,
style=alphabetic,
useprefix=true,
sorting=anyt,
maxbibnames=99
]{biblatex}
\usepackage{enumitem}
\usepackage{graphicx}

\usepackage{todonotes}

\newcommand{\err}{\mathcal{E}}

\newcommand{\res}{\mathrm{res.}}
\usepackage{lipsum}

\newcommand{\ddd}[1]{\Delta_{m, \sss}(\delta/#1)}
\newcommand{\vc}{\textsc{vc}}

\usepackage{wrapfig}
\usepackage{csquotes}

\title{Selective Classification via One-Sided Prediction}
\author{ Aditya Gangrade \\ {\small Boston University} \\ {\small \texttt{gangrade@bu.edu}} \and Anil Kag \\ {\small Boston University}\\ {\small \texttt{anilkag@bu.edu}} \and Venkatesh Saligrama \\ {\small Boston University} \\ {\small \texttt{srv@bu.edu}}}
\date{\vspace{-12pt}}

\begin{document}
\maketitle

\begin{abstract}
We propose a novel method for selective classification (SC), a problem which allows a classifier to abstain from predicting some instances, thus trading off accuracy against coverage (the fraction of instances predicted). In contrast to prior gating or confidence-set based work, our proposed method optimises a collection of class-wise decoupled one-sided empirical risks, and is in essence a method for explicitly finding the largest decision sets for each class that have few false positives. This one-sided prediction (OSP) based relaxation yields an SC scheme that attains near-optimal coverage in the practically relevant high target accuracy regime, and further admits efficient implementation, leading to a flexible and principled method for SC. We theoretically derive generalization bounds for SC and OSP, and empirically we show that our scheme strongly outperforms state of the art methods in coverage at small error levels.
\end{abstract}

\section{Introduction}
Selective Classification is a classical problem that goes back to the work of Chow \cite{Chow57, chow1970optimum}. The setup allows a learner to classify a query into a class, or to abstain from doing so (we also call this `rejecting' the query). This abstention models real-world decisions to gather further data/features, or engage experts, all of which may be costly. Such considerations commonly arise in diverse settings, including healthcare\footnote{For example, when deciding if a mammary mass is benign or malignant, a physician may predict based on ultrasound imaging tests, and, in more subtle cases, abstain and refer the patient to a specialist, or recommend specialised imaging such as CT scans.}, security, web search, and the internet of things (\cite{JMLR:v15:xu14a, zhu_iot}), all of which require very low error rates (lower even than the Bayes risk of standard classification). The challenge of SC is to attain such low errors while keeping coverage (i.e., the probability of not rejecting a point) high. This is a difficult problem because any choice of what points to reject is intimately coupled with the classifiers chosen for the remaining points.

The most common SC method is via `gating,' in which rejection is explicitly modelled by a binary-valued function $\gamma,$ and classification is handled by a function $\pi$. An instance, $x$, is predicted as $\pi(x)$ if  $\gamma(x) =1,$ and otherwise rejected. Within this formulation, recent work has proposed a number of methods, ranging from alternating minimisation based joint training, to the design of new surrogate losses, and of new model classes to accommodate rejection. {Despite this increased complexity, these methods lack power, as shown by the fact that they do not significantly outperform na\"{i}ve schemes that rely on abstaining on the basis of post-hoc uncertainty estimates for a trained standard classifier. This represents a significant gap in the practical effectiveness of selective classification.}
%Despite this increase in complexity, the SOTA methods so engendered lack power. This is illustrated by the fact that
%}% These na\"{i}ve scheme uses the SC structure very weakly, and so one expects an effective SC scheme to outperform it.} 

\noindent \textbf{Our Contributions.} We describe a new formulation for the SC problem, that comprises of directly learning \emph{disjoint} classification regions $\{\s_k\}_{k \in \mathcal{Y}}$, each of which corresponds to labelling the instance as $k$ respectively. Rejection is \emph{implicitly defined} as the gap, i.e., the set $\rr= \mathcal{X} \setminus \bigcup \s_k$. We show that this formulation is equivalent to earlier approaches, thus retaining expressivity.

The principal benefit of our formulation is that it admits a natural relaxation, via dropping the disjointness constraints, into \emph{decoupled} `one-sided prediction' (OSP) problems. We show that at design error $\varepsilon,$ this relaxation has the coverage optimality gap bounded by $\varepsilon$ itself, and so the relaxation is statistically efficient in the practically relevant high target accuracy regime. 

We pose OSP as a standard constrained learning problem, and due to the decoupling property, they can be approached by standard techniques. We design a method that efficiently adjusts to inter-class heterogeneity by solving a minimax program, controlled by one parameter that limits overall error rates. This yields a powerful SC training method that does not require designing of special losses or model classes, instead allowing use of standard discriminative tools. %Further, the method is flexible in that it can be used on top of existing representations of the data, or can be used to learn new problem-specific representations.

{To validate these claims, we implement the resulting SC methods on benchmark vision datasets - CIFAR-10, SVHN, and Cats \& Dogs. We empirically find that the OSP-based scheme has a consistent advantage over SOTA methods in the regime of low target error. In particular, we show a clear advantage over the na\"{i}ve scheme described above, which in our opinion is a significant first milestone in the practice of selective classification.}

\subsection{Related Work}%\label{sec:rw}

\textbf{State of the Art (SOTA) methods:} The SOTA, in terms of performance, for SC is encapsulated by three methods. The Na\"{i}ve method, i.e., rejecting when the output of a soft classifier is non-informative (e.g.~classifier margin is too small),
%if $\max_{k} f_k(x)$ is too small for a soft $K$-class classifier $f$), 
and this is surprisingly effective when implemented for modern model classes such as DNNs (\cite{geifman2017selective}). The only other methods that can (marginally) beat this are due to Liu et al., who design a loss function for DNNs \cite{deep_gamblers}, and to Geifman \& El-Yaniv, who design a new architecture for DNNs that incorporates gating \cite{geifman2019selectivenet}. 

The methods of \cite{deep_gamblers, geifman2019selectivenet} are both based on the \textbf{Gating formulation}, mentioned earlier. This formulation was popularised by Cortes et al. \cite{cortes2016learning}, although similar proposals appeared previously \cite{el2010foundations, wiener2011agnostic}. A number of papers have since extended this approach, e.g.~designing training algorithms via alternating minimisation \cite{nan2017adaptive, nan2017dynamic}, designing loss functions \cite{deep_gamblers, Ni_calibration_CWR, ramaswamy2018}, and model classes, such as an architecturally augmented deep neural network (DNN) \cite{geifman2019selectivenet}. In contrast, our work develops an alternate formulation that directly solves SC without use of specialised losses or model classes.

The na\"{i}ve method has its roots in the \textbf{Direct SC} formulation, which is based on learning a function $f: \mathcal{X} \to \{1,\dots,K, ?\}$ (where $?$ denotes rejection), and is pursued by Wegkamp and coauthors \cite{herbei-wegkamp, bartlett2008classification, wegkamp2007lasso, wegkamp-yuan-svm, Wegkamp-Yuan-Convex}. The main disadvantage of this formulation is that the methods emerging from it consider very restricted forms of rejection decisions, e.g.~ $\{|\phi - \nicefrac{1}{2}|< \delta \}$, where $\phi$ is a softmax output of a binary classifier.% That said, recent SOTA work like \cite{deep_gamblers}, while ostensibly gating-based, is related to the direct trichotomy in that a single function with soft outputs $(\phi_+, \phi_-, \phi_?)$ is learnt, from which hard rejection decisions are obtained by thresholding $\phi_?.$

Rather than including an explicit gate, our formulation and method for learning abstaining classifiers uses an \textbf{implicit abstention criterion}, by modelling regions of high confidence directly. Such an approach was theoretically considered by Kalai et al.~\cite{kalai2012reliable} for the binary setting, although an implementable methodology was not developed from the same. This paper also suggests a decoupled approach to learning. Independently and concurrently of our work, Charoenphakdee et al.~\cite{charoenphakdee2021classification} also propose an implicitly gated method by observing that in the situation where abstention has a fixed cost, the Bayes optimal classifier can be derived using a cost-sensitive objective. They develop this into a methodology for learning selective classifiers that bears significant commonalities to ours in the structure of the losses constructed and approach taken, although their exploration is focused on the situation with a fixed cost for abstention (the so-called Chow loss). Together these papers suggest that the approach we design can be motivated in multiple ways.\footnote{NB. This paragraph was added in version 4 of this paper, to describe prior related work that we were unaware of in the first case, and to describe concurrent related work that bears methodological similarities.}

An alternate \textbf{Confidence Set formulation} (which also features an implicit abstention criterion) has been pursued in the statistics literature \cite{Lei_class_w_confidence, denis-hebiri} (for the binary case), and involves learning sets $\{\cc_k\}_{k \in [1:K]}$ such that $\bigcup \cc_k= \mathcal{X},$ and each $\cc_k$ covers class $k$ in the sense $\P(\cc_k|Y = k)$ is large\footnote{More accurately, this precise formulation has not appeared for the multiclass setting, and only appears for the binary problem in \cite{Lei_class_w_confidence, denis-hebiri}. Here we are expressing the natural multiclass extension of this, that turns out to be equivalent to selective classification (\S\ref{sec:post_form_comparison}). The existing literature instead pursues the multiclass extension to LASV classification, as mentioned above. Please see \S\ref{appx:confidence_set_rant} for a detailed discussion}. Points which lie in two or more of the $\cc_k$s are rejected, and otherwise points are labelled according to which $\cc_k$ they lie in. While this has subsequently been extended to the multiclass setting \cite{LeiWasserman,  denis_hebiri_multi, denis_hebiri_chzhen_minimax}, these papers study `least ambiguous set-valued classification', which is a different problem from selective classification and does not express it well (see \S\ref{appx:confidence_set_rant}). A limitation of existing work in this framework is their reliance on estimating the regression function $\eta(x):= \P(Y = k|X = x)$ to ensure efficiency. Proposals typically go via using non-parametric estimates of $\eta$, which are then filtered. On a practical level, this reliance on estimation reduces statistical efficiency, and on a principled level, this violates Vapnik's maxim of avoiding solving a more general problem as an intermediate step to solving a given problem (\cite[\S1.9]{vapnikbook}). 

While our formulation is most closely related to the confidence set formulation, and is equivalent to a change of variables of this (\S\ref{sec:post_form_comparison}), it is directly motivated. Furthermore, our framework naturally leads to relaxations to OSP that let us study discriminative methods on high-dimensional datasets and large model classes, which are unexplored in these works.

In passing, we mention the \emph{uncertainty estimation} (UE), and \emph{budget learning} (BL) problems. UE involves estimating model uncertainty at any point \cite{gal2016dropout, lakshminarayanan2017simple}, which can plug into both na\"{i}ve classifiers, and the other methods. As such, UE is a vast generalisation of SC.  BL is a restricted form of SC that aims at reaching the accuracy of a complex model using simple functions, and is relevant for efficient inference constraints. %(\cite{gangrade2020budget}).

We highlight a recent \emph{decoupling-based} method for BL that involves the first and last authors \cite{gangrade2020budget}. The present work can be seen as considerable extension of this paper to full SC. While the broad strategies of decoupling schemes are similar, significant differences arise since much of the structure developed in the prior work does not generalise to SC, and development of new forms is necessary. Additionally, our experiments study large multiclass models going beyond best achievable standard accuracy, while the previous work only studies small binary models getting to standard accuracy achievable by larger models. 
\vspace{0pt}
\section{Formulation and Methods}\label{sec:form}
\vspace{0pt}
\noindent \textbf{Notation.} Probabilities are denoted as $\P$, random variables are capitalised letters, while their realisations are lowercase ($X$ and $x$). Sets are denoted as calligraphic letters, and classes of sets as formal script ($\s \in \sss$). Parameters are denoted as greek letters. For a set $\mathcal{S} \subset \mathcal{X},$ $\mathbb{P}(\mathcal{S})$ is shorthand for $\mathbb{P}(X \in \mathcal{S})$. 

We adopt the supervised learning setup - data is distributed according to an unknown joint law $\P$ on $\mathcal{X} \times \mathcal{Y}$, and we observe $n$ i.i.d.~points $(X_i, Y_i) \sim \P$. For $K$ classes, we set $\mathcal{Y} = [1:K],$ where $K$ is a constant independent of $|\mathcal{X}|$. We use $\sss$ to denote the class of sets from which we learn classifiers.%\textcolor{red}{To allow feasibility of the various programs, we make the (very mild) assumptions that $\varnothing \in \sss$, and further that for any choice of $K-1$ pairwise disjoint sets $\mathcal{A}_1,\dots, \mathcal{A}_{K-1} \in \sss,$ it holds that $\left(\bigcup_{j = 1}^{K-1} \mathcal{A}_j\right)^c \in \sss$. This amounts to demanding that $\sss$ can express a $K$-ary classifier of the form `label points in $\mathcal{A}_k$ as $k$ for $k\le K-1$, and label the rest $K$,' and we thus call it the $K$-classifier condition.}% Measure theoretic considerations will be suppressed throughout, as is common in learning theory.

\subsection{Formulation of SC}

We set up the SC problem (Fig.~\ref{fig:comparison_of_formulations}(top) illustrates binary case) as that of directly recovering disjoint classification regions, $\{ \s_k\}_{k \in [1:K]}$ from a class of sets $\sss,$ under the constraint that the error rate is smaller than a given level $\varepsilon$, which we call the target error. Each such $K$-tuple of sets induces two events of interest - the rejection event, and the error event. \begin{align*} 
    \rr_{\{\s_k\}} &:= \left\{ X \in \left(\bigcup \s_k\right)^c \right\} \\
    \err_{\{\s_k\}} &:= \bigcup \{ X \in \s_k, Y \neq k \}.
\end{align*} 
We will usually suppress the dependence of $\rr, \err$ on $\{\s_k\}.$ Notice further that $\err$ decomposes naturally into events that depend only on one of the $\s_k$s. We will call these `one-sided' error events \[ \err^k_{\s_k} = \{X \in \s_k, Y \neq k\}. \]
With the above notation, we pose the problem as a maximisation program. The value of this is said to be the \emph{coverage at target error level $\varepsilon$}, denoted  $C(\varepsilon; \sss)$.\begin{align}\label{SC} 
C(\varepsilon; \sss) = \max_{ \{\s_k\}_{k \in [1:K]} \in \sss}  & \sum_{k = 1}^K \P(\s_k) \tag{SC}\\ \textrm{s.t.}\quad&\P(\err_{\{\s_k\}}) \leq \varepsilon, \notag \\ &\P(\bigcup_{k,k'\neq k} \s_k \cap \s_{k'}) = 0,\notag  
\end{align} where the final constraint is expressing the fact that the $\s_k$s must be pairwise disjoint. Note that if $\varepsilon$ equals the Bayes risk of standard classification with $\sss,$ then (\ref{SC}) recovers the standard solution and coverage $1$.

\emph{Example.} Consider the case of $K = 2$ where $\P_X$ is uniform on $[0,1],$ $\P(Y = 1|X = x) = x,$ and $\sss$ consists of single threshold sets $\{ x > t\}, \{x \le t\}$ for $t \in [0,1]$. The Bayes risk of standard classification is $\nicefrac{1}{4}$. For any $\varepsilon < \nicefrac{1}{4},$ the coverage at level $\varepsilon$ is $C(\varepsilon;\sss) = 2\sqrt{\varepsilon}$, which is attained by $\s_1 = \{ x > 1 - \sqrt{\varepsilon}\}, \s_2 = \{x \le \sqrt{\varepsilon}\}$.

\subsubsection{Design choices}\label{sec:alternate_goals}
We outline alternate ways to set up the SC problem that we don't pursue in this paper.

\emph{Form of constraints.} In (\ref{SC}), we maximise coverage, while controlling error, which is \emph{error-constrained SC}. Alternately one can pursue the equivalent \emph{coverage constrained SC} problem - minimising $\P(\err)$ subject to $\P(\rr) \le \varrho.$ 

As illustrated in the starting example, our interest in SC is driven by the desire to attain very small error rates. We thus find the error constrained form of SC more natural, and since we needed to select one of the two for the sake of brevity, we adopt it in the rest of the paper.\footnote{This is not to imply that the coverage constrained form cannot be more appropriate for some settings. Which one to use in practice is ultimately a problem specific choice.} We note that our method is also effective for coverage-constrained SC, as shown empirically in \S\ref{sec:exp}.

\emph{Error criterion.} In (\ref{SC}), we constrain the raw error $\P(\err).$ This has the benefit of being both natural, since it directly controls the standard error metric, and further, simple. Alternate forms of the error metric have been studied in the literature - e.g.  conditioning on acceptance ($P(\err|\rr^c)$) \cite{geifman2019selectivenet}; and separately constrained class conditionals ($P(\err|Y = k) \le \varepsilon_k$) \cite{Lei_class_w_confidence}. Most of the development below can be adapted to these settings with minimal changes, and we restrict attention to $\P(\err)$ for concreteness.

\subsection{Relaxation and One-sided Prediction}\label{sec:OSP}

(\ref{SC}) couples the $\s_k$s via the $\P$-a.s. disjointness constraint. We now develop a decoupling relaxation.

To begin, note that we may decouple the error constraint by introducing variables that trades off the one-sided error rates as below. This program is equivalent to (\ref{SC}) in the sense that they have the same optimal value, and the same $\{\s_k\}$ achieve this value. \begin{align} \label{SC:expanded}
\max_{\substack{\{\s_k\} \in \sss, \{\alpha_k\}\in [0,1]}}& \sum_{k = 1}^K\P(\s_k) \tag{SC-expanded}\\ \textrm{s.t. }\quad &\forall k: \P(\err^k_{\s_k} ) \leq \alpha_k \varepsilon,\quad \sum \alpha_k \le 1, \notag \\ &\P(\bigcup_{k,k'\neq k} \s_k \cap \s_{k'}) = 0.\notag 
\end{align}

Our proposed relaxation is to simply drop the final constraint. The resulting program may be decoupled, via a search over the variables $\alpha_k$ into $K$ \emph{one-sided prediction} (OSP) problems:\begin{align}
L_k(\varepsilon_k; \sss) &= \max_{\substack{\s_k \in \sss} }  \P(\s_k) \textrm{ s.t. } \P(\err^k_{\s_k}) \leq \varepsilon_k  \tag{OSP-k} \label{OSPk}
\end{align} Notice that the above OSPk problem demands finding the \emph{largest} set $\s_k$ that has a low false alarm probability for the null hypothesis $Y \neq k$. Structurally this is the opposite to the more common Anomaly Detection problem, which demands finding the smallest set with a low missed detection probability. 

We note that while we decouple the SC problem completely above, the main benefit is the removal of the intersection constraint, which is the principal difficulty in SC. The sum error constraint is benign, and for reasons of efficiency we will reintroduce it in \S\ref{sec:prac_meta}.

Continuing, observe that the sets recovered from the above problems may overlap, which introduces an ambiguous region. This overlap region is necessarily of small mass (Prop.~\ref{prop:OSP_not_lossy}), and so may be dealt with in any convenient way. Theoretically we break ambiguities in the favour of the smallest label. These sets need not belong to $\sss$ anymore, and so this is an (weakly) improper classification scheme. 

Overall this gives the following infinite sample scheme:
\begin{itemize}[leftmargin = 0.15in, nosep]
    \item For each feasible $\alpha \in [0,1]^K,$ solve for $\{L_k(\alpha_k \varepsilon)\}$ for each ${k \in [1:K]}$. Let $\{\mathcal{T}_k^\alpha\}$ be the recovered sets.
    \item Let $\s_k^{\alpha} = \mathcal{T}_k^{\alpha} \setminus \left(\bigcup_{k'< k} \mathcal{T}_{k'}^{\alpha}\right)$.
    \item Return the $\{\s_k^\alpha\}$ that maximises $\sum_k \P(\s_k^\alpha)$ over $\alpha$.
\end{itemize}

At small target error levels, which is our intended regime of study, the resulting sets are guaranteed to not be too lossy, as in the following statement. The above is shown (in \S\ref{appx:prop1pf}) by arguing that the mass of the overlap between the OSP solutions (the $\mathcal{T}_k$) is at most $2\varepsilon.$ Empirically this is even lower, see Table \ref{tab:percentage-overlap-between-two-one-sided-learners}. \begin{myprop}\label{prop:OSP_not_lossy} 
If $\{\s_k\}$ are the sets recovered by the procedure above, then these are feasible for \emph{(\ref{SC})}. Further, their optimality gap is at most $2\varepsilon$, i.e.~\[  \sum_{k \in [1:K]} \P(\s_k) \ge C(\varepsilon; \sss) - 2\varepsilon.\]
\end{myprop} 
% \noindent \textbf{Multiclass relaxation} The same relaxation (i.e., dropping the disjointness constraints) can also be applied to the multiclass setting, leading to analogous OSP problems for each class with values $L_k(\varepsilon_k, \sss)$. The overlaps remain of small size, and the analogue of Prop.~\ref{prop:OSP_not_lossy} has the same optimality gap.
\subsection{Equivalence of SC formulations}\label{sec:post_form_comparison}

We show that the prior gating and confidence frameworks are equivalent to ours, based on transforming feasible solutions of one framework into an other.

\emph{Gating}: Denote the acceptance set of gating as $\Gamma = \{ \gamma = 1\},$ and let the predictions be $\Pi_k = \{ \pi = k\}$. Taking $\s_k = \Pi_k \cap \Gamma$ yields disjoint sets that can serve for SC under our formulation that have the same decision regions for each class, and the same rejection region, since $(\bigcup \s_k)^c = \Gamma^c$. Conversely, for disjoint decision sets $\s_k,$ the gate $\Gamma = \bigcup \s_k,$ and the predictor $\Pi_k = \s_k$ form the corresponding gating solution.

\emph{Confidence set}: Take confidence sets $\{\cc_k\}$ which cover $\mathcal{X},$ and have the rejection set $\mathcal{B} = \bigcup_{k \neq k'} \cc_k \cap \cc_{k'}$. Then we produce the disjoint sets $\s_k = \cc_k\setminus (\bigcup_{k' \neq k} \cc_{k'}),$ which retain the same decision regions. These also have the same rejection region because we may express $\s_k = \cc_k \cap \mathcal{B}^c$, and thus $\bigcap \s_k^c = (\bigcap_k \cc_k^c) \cup \mathcal{B}$, and $\bigcap \cc_k^c = \varnothing$ since the $\cc_k$ cover the space. Conversely, for disjoint $\{\s_k\}$, the sets $\cc_k = (\bigcup_{k' \neq k}\s_{k'})^c = \s_k \cup \rr$ cover the space, and have the rejection region $\rr$ since $\cc_k \cap \cc_{k'} = \rr$ for any pair $k\neq k'$. 

Figure \ref{fig:comparison_of_formulations} illustrates these equivalences. Notice that due to the simplicity of the reductions, these equivalences are fine-grained in that the joint complexity of the family of sets used is preserved in going from one to the other. 
\begin{figure}[b]
    \label{fig:comparison_of_formulations}
    \centering 
    \includegraphics[width = 0.32\textwidth]{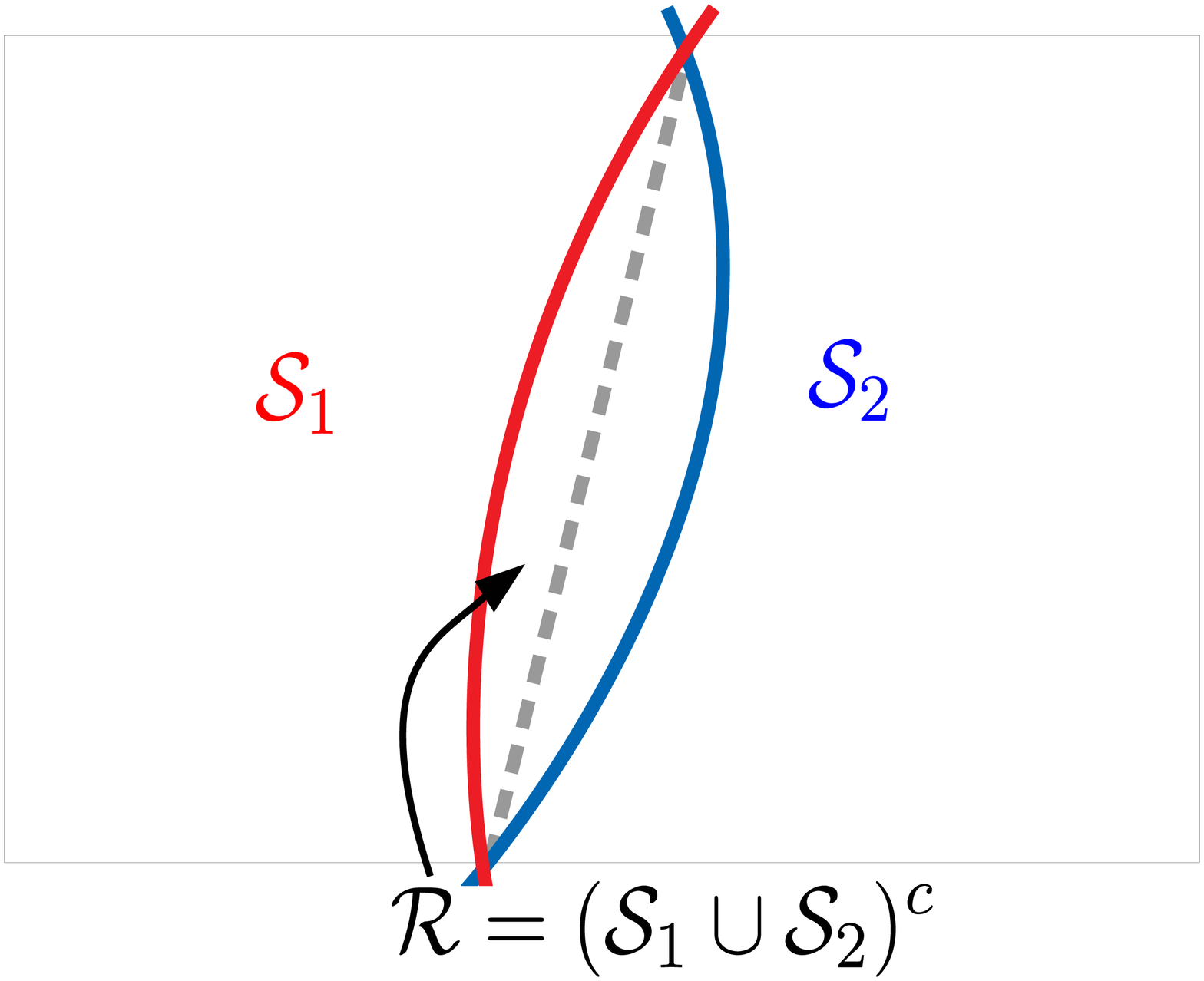}~
    \includegraphics[width=0.32\textwidth]{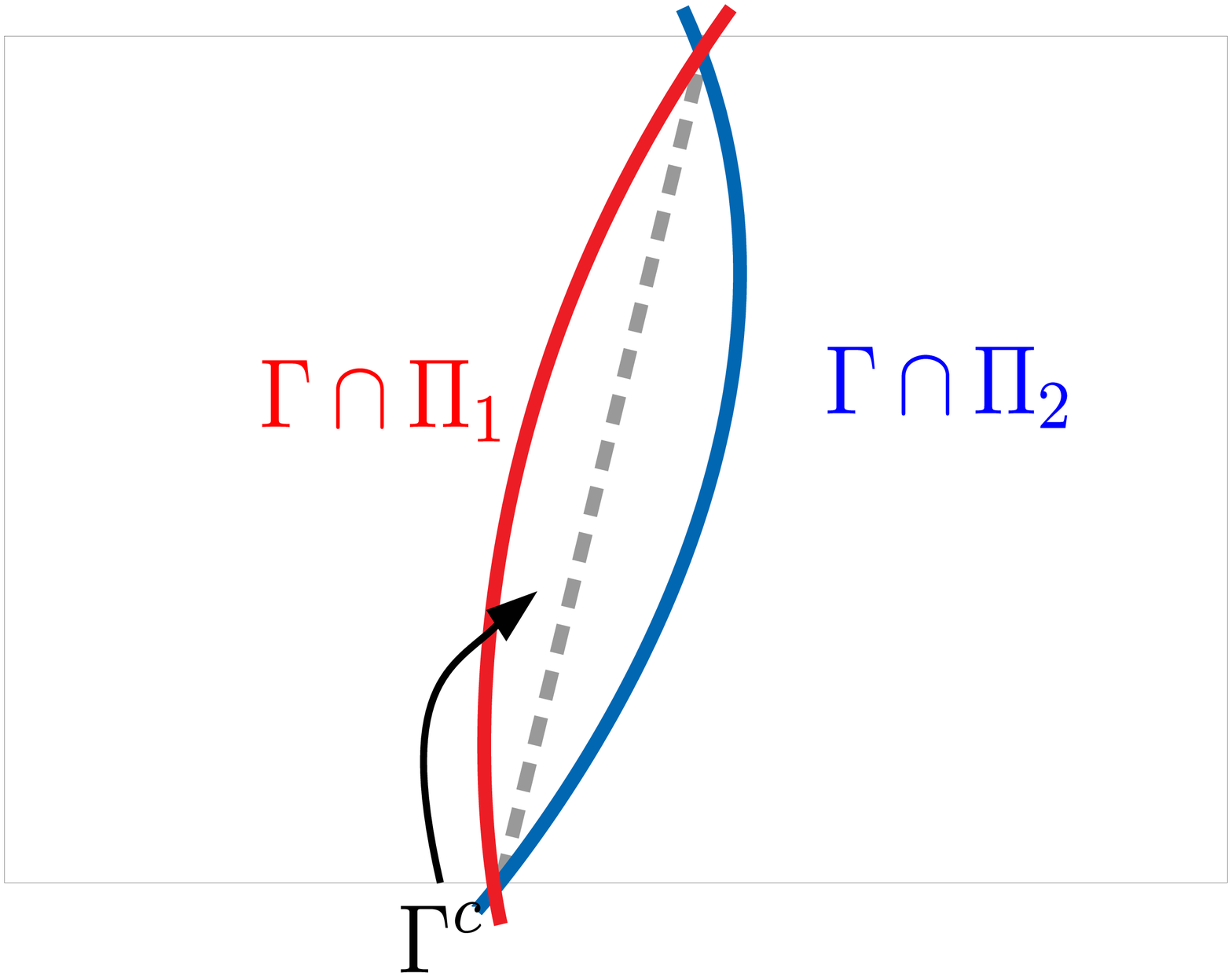}~    ~
    \includegraphics[width = 0.32\textwidth]{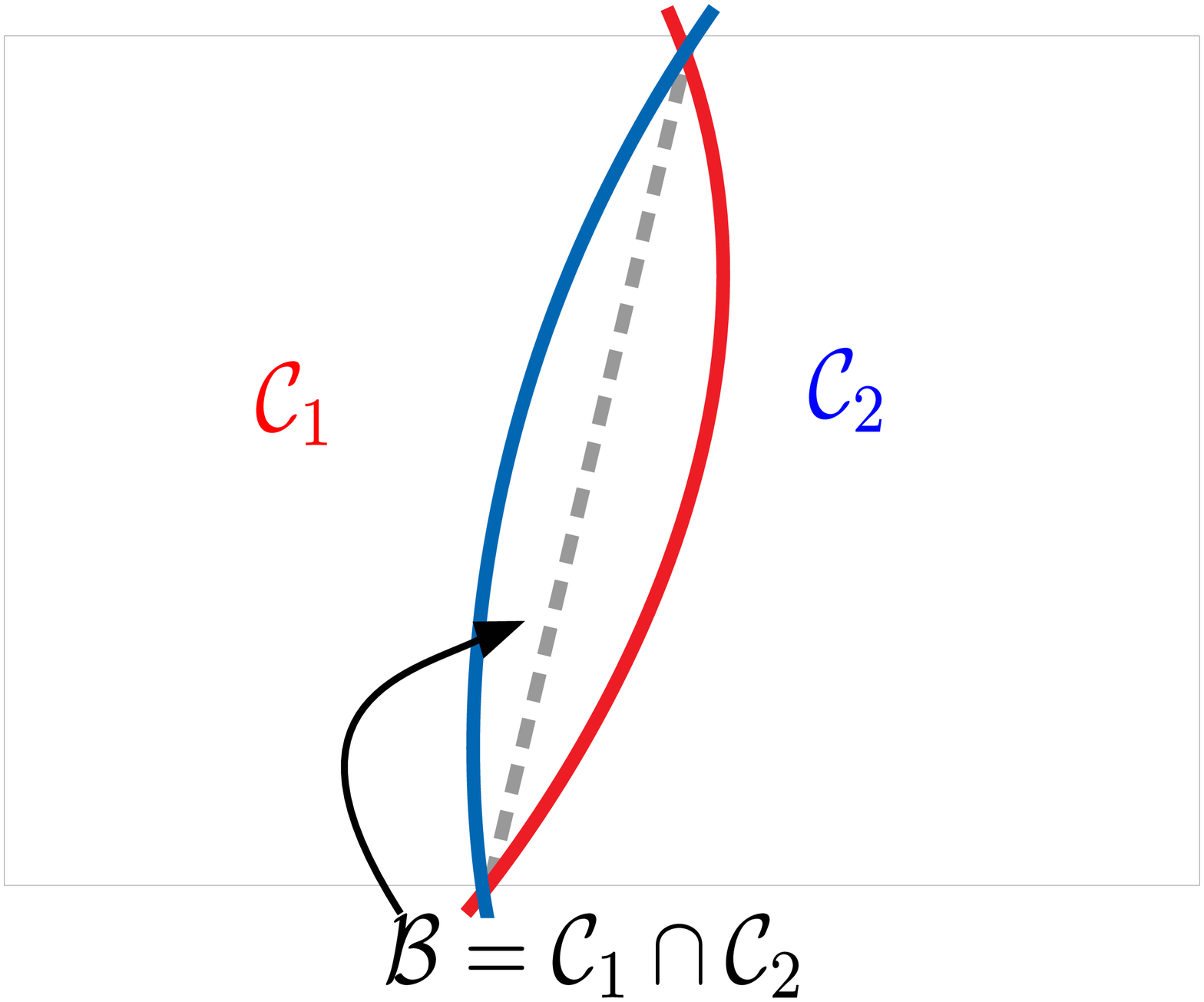} 
    \caption{ An illustration of the equivalence between the three formulations for binary classification. \textit{Left:} our formulation; $\s_i$ denotes disjoint sets; \textit{Middle:} gating with $\Gamma$ representing gated set; \textit{Right:} ${\cal C}_i$ represents confidence sets, and their intersection representing the rejected set. In each case, the coloured curves represent the boundary of the set labelled with the corresponding colour, and the dashed line is the Bayes boundary.}
\end{figure}
Given these equivalences, we again distinguish our approach from the existing ones. 

First, the structure of solutions is markedly different. The gating formulation takes both the rejection and the classification decisions explicitly via the two different sets. The confidence set formulation takes neither explicitly, and instead produces a `list decoding' type solution. In contrast, we make the classification decisions explicit, and produce the rejection decision implicitly.

Consequently, the salient differences lies in the method. The gating based methods have concentrated on the design of surrogate losses and models, while for the confidence set, methods either go through estimating the regression function, or via a reduction to anomaly detection type problems \cite{LeiWasserman, denis_hebiri_multi}\footnote{We refer to \S\ref{appx:our_method_is_diff} for a deeper discussion of this point}. In strong contrast, we develop a new relaxation that allows decoupled learning via `one-sided prediction' problems. These OSP problems are almost opposite to anomaly detection - instead of finding small sets for each class that do not leave too much of its mass missing, we instead learn large sets that do not admit too much of the complementary class' mass.

\subsection{Finite Sample Properties of OSP} \label{sec.theory}

Thus far we have spoken of the full information setting. This section gives basic generalisation analyses for an empirical risk minimisation (ERM) based finite sample approach. Since the one-sided problems are entirely symmetric, we concentrate only on OSP-1, that is (\ref{OSPk}) with $k = 1,$ below. Note that the SC problem can directly be analysed in a similar way (Appx.~\ref{appx:SC_gen}), but we focus on the OSP problem, since this underlies the method we pursue.

We show asymptotic feasibility of solutions, that is, we show that we can, with high probability, recover a set $\s$ for OSP such that $\P(\s) \ge L_1(\varepsilon) - o(1)$ and $\P(\err^1_\s) \le \varepsilon + o(1)$, where the $o$ are as the sample size diverges. This is in contrast to exact feasibility, i.e., insisting on $\s'$ such that $\P(\err^1_{\s'}) \le \varepsilon$ with high probability. Exactly satisfying constraints via ERM whilst maintaining that the objective is also approaching the optimum is a subtle problem, and has been shown to be impossible in certain cases \cite{rigollet2011neyman}. On the other hand, plug-in methods along with an `identifiability' condition which imposes that the law of $\eta(X)$ is not varying too fast at any point can be employed to give exact constraint satisfaction along with a small excess risk - the technique was developed by Tong \cite{tong_plugin}, and has been used previously in SC contexts \cite[e.g.,][]{shekhar2019binary}. However, since the applicability of plug-in methods to large datasets in high dimensions is limited, we do not pursue this avenue here.%\textcolor{red}{clarify, also throw in bit about surrogatified programs?}

\subsubsection*{One-Sided Learnability}
\begin{defi}
We say that a class $\sss$ is one-sided learnable if for every $\varepsilon \ge 0$ and $(\delta, \sigma, \nu) \in (0,1)^3,$ there exists a finite $m(\delta, \sigma, \nu)$ and an algorithm $\mathfrak{A}: (\mathcal{X} \times [1:K])^m \to \sss$ such that for any law $\P,$ given $m$ i.i.d.~samples from $\P$, $\mathfrak{A}$ produces a set $\s_1 \in \sss$ such that with probability at least $1-\delta$ over the data, \[  \P(\s_1) \ge L_1(\varepsilon; \sss) - \sigma, \qquad \textrm{and}\qquad \P(\err^1_{\s_1}) \le \varepsilon + \nu.\]
\end{defi}
The characterisation we offer is 
\begin{myprop}\label{prop:OSP_vc}
 A class $\sss$ is one-sided learnable iff it has finite VC dimension. In particular, given $n$ samples, we can obtain a set $\s_1$ that, with probability at least $1-\delta$, satisfies 
 \begin{alignat*}{3} 
    &\P(\s_1) \ge L_1(&&\varepsilon; \sss)  &&- \sqrt{C_K\frac{ (\textsc{vc}(\sss)\log n + \log(C_K/\delta))}{n}}\\
    &\P(\err^1_{\s_1}) \le &&\varepsilon  && + \sqrt{C_K\frac{ (\textsc{vc}(\sss)\log n + \log(C_K/\delta))}{n}},
\end{alignat*} 
where $C_K$ is a constant that depends only on the number of classes $K$.
\end{myprop}

The proof of the necessity of finite VC dimension is via a reduction to standard learning, while the upper bounds on rates above follow from uniform convergence due finite VC dimension. See \S\ref{appx:samplecomp_pfs}. The scheme attaining these is a direct ERM that replaces all $\P$s in (\ref{OSPk}) by empirical distributions.
%
%The scheme attaining this is a direct constrained ERM, i.e., the program \begin{align}
%\widehat{L}_+(\varepsilon; \sss) = \max_{\substack{\s_1 \in \sss} } \widehat{\mathbb{P}}(\s_1) \quad\textrm{s.t.} \quad \widehat{\mathbb{P}}(X \in \s_1, Y = 2) \leq \varepsilon,     \tag{Emp-OSP+}     \label{OSP+emp}
%\end{align}%
%
%where $\widehat{\mathbb{P}}$ denotes the empirical distribution induced by the examples. 

On the whole, applying the above result for each of the $K$ OSP problems tells us that if we can solve the empirical OSP problems for the indicator losses and constraints, then we can recover a SC scheme that, with high probability, incurs error of at most $\varepsilon + O(1/\sqrt{n})$ and has coverage of at least $C(\varepsilon;\sss) - 2\varepsilon - O(1/\sqrt{n}).$

% \begin{myprop}
% A class $\sss$ is learnable with abstention if and only if it has finite VC dimension. In particular, the sample complexities satisfy \begin{align*}
%     n(\delta, \zeta, \sigma, \nu) &\le K\textsc{vc}(\sss) \log(K/\delta) \left( \frac{1}{\nu} + \frac{1}{\sigma^2} \right) \\
%     m(\delta, \zeta, \sigma, \nu) &\le K \frac{\textsc{vc}(\sss) \log(K/\delta)}{\zeta^2}, 
% \end{align*}
% where $K \le 40$ is a constant not depending on any parameters. 
% \end{myprop}

\section{Method}\label{sec:prac_meta}
In this section, we derive an efficient scheme, first by replacing indicator losses with two differentiable surrogate variants, and then propose OSP relaxations. A summary of the method expressed as pseudo-code is included in Appx.~\ref{appx:algorithmics}. Throughout, $\sss$ is set to be level sets of the soft output of a deep neural network (DNN), i.e., $\sss = \{f(\cdot;\theta) > t\}$, where $f(\cdot;\theta): \mathcal{X} \to [0,1]$ is a DNN parametrised by $\theta$. The bulk of the exposition concerns learning $\theta$s. In this and the following section, $\{(x_i, y_i)\}_{i = 1}^n$ refers to a training dataset with $n$ labelled data points.

\noindent \textbf{Relaxed losses.} To solve the OSP problem, we follow the standard approach of replacing indicator losses by differentiable ones. This sets up the relaxed problem \begin{equation*}  \min_{\theta_k} \frac{\sum_{i} \ell(f(x_i; {\theta_k}))}{n} \textrm{ s.t.} \frac{\sum_{i: y_i \neq k} \ell'(f(x_i;{\theta_k}))}{n_{\neq k}} \le \varphi_k \end{equation*} where $\theta_k$ parametrises the DNN, $\varphi_k$ denote relaxed values of the constraints, and $\ell, \ell'$ are surrogate losses that are small for large values of their argument, and $n_{\neq k} = |\{i:y_i \neq k\}|$. In the experiments we use $\ell(z) = -\log(z)$ and $\ell'(z) = -\log(1-z)$, essentially giving a weighted cross entropy loss. We refer to the objective of the above problem as $\widetilde{L}_k(\theta_k)$, and the constraint as $\widetilde{C}_k(\theta_k)$.

\noindent \emph{A more stable loss.} Practically, the loss $\widetilde{L}_k$ suffers from instability due to the fact that the first term sums over all instances. This can seen clearly when $\ell = -\log,$ for which the objective includes the sum $\sum_{i: y_i \neq k} -\log(f(x_i;{\theta_k}))$. Since for negative examples we expect $f(x;{\theta_k})$ to be small, this sum is very sensitive to perturbations in these values, which reduces the quality of the solutions. To ameliorate this, we formulate the following `restricted' loss, where the objective instead sums over only the positively labelled samples \begin{equation}\label{OSP_nn_restrict} 
   \min_{\theta_k} \frac{\sum_{i : y_i = k} \ell(f(x_i;{\theta_k}))}{n_k} \textrm{ s.t. } \widetilde{C}_k(\theta_k) \le \varphi_k.
\end{equation}
Notice that the constraint $\widetilde{C}_k$ is the same as before. We refer to the restricted objective above as $\widetilde{L}_k^{\res}(\theta_k).$ This loss underlies all further methods, and $\S\ref{sec:exp}$.

Note that the above program remains sound w.r.t. the OSP task, since it is a surrogate for the following \[ \max_{\s_k \in \sss} \P(X \in \s_k, Y = k) \textrm{ s.t. } \P(X \in \s_k, Y \neq k) \le \varepsilon.\] Comparing (\ref{OSPk}) and the above, the constraints are the same, and the objectives differ by $\P(\s_k) - \P(\s_k, Y = k) = \P(\s_k, Y \neq k),$ which, due to the constraint, is at most $\varepsilon$. Thus, the programs are equivalent up to a small gap (that is, optimal solutions for the above attain a value for (\ref{OSPk}) that is $\varepsilon$-close to the optimal value for it). For the same reason, we can use the solutions of the above one-sided problem in the scheme of \S\ref{sec:OSP} to yield solutions feasible for (\ref{SC}) that satisfy an analogue of Prop.~\ref{prop:OSP_not_lossy} with an optimality gap of $3\varepsilon$ instead of $2\varepsilon$. %Empirically, we find that neither loss strictly dominates the other (Table \ref{table:main_results_merged}).

\noindent \textbf{Joint Optimisation and normalisation.} A na\"{i}ve approach with the above relaxations in hand is to optimise the $k$ OSP problems separately. However, this leads to an exponential in $K$ rise in complexity in the model selection process, since different values of $(\varphi_1, \dots, \varphi_K)$ need to be selected - if $\Phi$ such values are searched over for each $\varphi_k,$ then this amounts to a prohibitive grid search over $\Phi^K$ values. In addition, due to class-wise heterogeneity, the values of $\varphi_k$s need not be calibrated across programs, and thus simple solutions like pinning all the $\varphi_k$s to the same value are not viable. A final issue is that a na\"{i}ve implementation of this setup results in training $K$ separate DNNs, which leads to a $K$-fold increase in model complexity.

We make two modifications to handle this situation. First, we normalise function outputs by adopting the following architecture: we consider DNNs with $K$ output nodes, each representing one of the $f_k$. The backbone layers of the network are shared across all OSP problems. Further, we take \[ f(x) = (f_1(x),\dots, f_K(x)) = \mathrm{softmax}(\langle w_k, \xi_\theta(x)\rangle),\] where $\xi_\theta$ denotes the backbone's output, and recall that $(\mathrm{softmax}(v))_k = \nicefrac{\exp{v_k}}{\sum \exp{v_k}}$. This normalisation and restricted model handles both the class-wise heterogeneity, and the blowup in model complexity.

For the sake of succinctness, we define $\mathbf{w} = (w_1, w_2, \dots, w_K),$ and $\boldsymbol{\varphi} = (\varphi_1, \dots, \varphi_K)$.

Next, in order to ameliorate the search, we propose jointly optimising the various OSP problems, by enforcing a joint constraint on the sum of the various constraint values via a single value $\varphi$. This mimics the structure of (\ref{SC}), where the constraint limits the sum of the one-sided errors. The relaxation thus amounts to dropping the disjointness constraint, and softening the indicators in (\ref{SC}). The resulting problem is \begin{align}\label{program:OSP_together}
    &\min_{\theta, \mathbf{w}, \boldsymbol{\varphi}} \sum_k \widetilde{L}_k^{\res}(\theta, \mathbf{w})  \\
    &\mathrm{s.t.} \quad \forall k: \widetilde{C}_k(\theta, \mathbf{w}) \le \varphi_k, \quad \sum \varphi_k \le \varphi, \notag
\end{align}
where recall $\widetilde{L}^{\res}, \widetilde{C}_k$ from above, which are functions of $(\theta, \mathbf{w})$ since the backbone $\theta$ is shared, and since all $f_k$ depend on all $w_k$s due to the softmax normalisation.

Finally, we propose optimising (\ref{program:OSP_together}) via stochastic gradient ascent-descent. We note that one tunable parameter - $\mu$ - remains in the problem, corresponding to the sum constraint on the $\varphi_k$s, while $\lambda_k$s are multipliers for the $\widetilde{C}_k$ constraints. We again denote $\boldsymbol{\lambda} = (\lambda_1, \dots, \lambda_K)$. The resulting Lagrangian is \begin{align} &\widetilde{M}^{\res}( \theta, \mathbf{w}, \boldsymbol{\varphi}, \boldsymbol{\lambda}, \mu) \\= &\sum_k  \widetilde{L}_k^{\res}(\theta, \mathbf{w}) + \lambda_k(\widetilde{C}_k(\theta, \mathbf{w}) - \varphi_k) + \mu \varphi_k,\notag \end{align} and we solve the problem 
\begin{align}\label{final_minimax_problem} \min_{ (\theta, \mathbf{w}, \boldsymbol{\varphi})} \max_{ \boldsymbol{\lambda}: \forall k, \lambda_k \ge 0} \widetilde{M}^{\res}(\theta, \mathbf{w}, \boldsymbol{\varphi}, \boldsymbol{\lambda}, \mu), \end{align} treating $\mu$ as the single tunable parameter.

We note that the Lagrangian above bears strong resemblance to a one-versus-all (OVA) multiclass classification objective. The principal difference arises from the fact that the losses are weighted by the $\lambda_k$ terms, and the optimisation trades these off, which are typically not seen in one-versus-all approaches (of course, we also use the resulting functions very differently).

\noindent \textbf{Thresholding and resulting SC solution.} The outputs of the classifiers learned with any given $\mu$ yield soft signals for the various OSP problems. To harden these into a decision, we threshold the outputs of the soft classifier at a common level $t \in [0,1]$. This crucially relies on the earlier normalisation of the soft scores to make them comparable. Finally, to deal with ambiguous regions, we use the soft signals $f_k$, and assign the label to the one with the largest score. Overall, this leads to the SC solution \begin{align}\label{SC_soln_from_params} \s_k(\theta, \mathbf{w},t) = &\{x : f_k(x;\theta, \mathbf{w}) \ge t\} \\ &\cap\,\, \{x: k = \argmax_{k'} f_{k'}(x;\theta, \mathbf{w}) \}. \notag\end{align}
\noindent \textbf{Model Selection.} The above setup has two scalar hyperparameters -  $\mu$ from (\ref{final_minimax_problem}), and threshold $t$ at which hard decisions are produced in (\ref{SC_soln_from_params}), and each choice of these yields a different solution. Our final model is one that performs the best on the validation dataset among all hyperparameter tuples ($\mu,t$). Concretely, let $\widehat{\mathbb{P}}_{V}$ denote the empirical law on a validation dataset. Denote the solutions from (\ref{final_minimax_problem}) with a choice of $\mu$ as $(\theta(\mu), \mathbf{w}(\mu))$. Let $\mathbf{M,T}$ respectively be discrete sets of $\mu$'s and $t$'s. The procedure is \begin{itemize}[nosep]
    \item For each $(\mu, t) \in \mathbf{M} \times \mathbf{T},$ and each $k,$ compute $\s_k(\mu,t) =  \s_k( \theta(\mu), \mathbf{w}(\mu), t)$ as defined in (\ref{SC_soln_from_params}).
    \item For each $(\mu,t) \in \mathbf{M} \times \mathbf{T},$ evaluate $\widehat{C}_V(\mu, t) = \sum_k \widehat{\P}_V(\s_k(\mu,t)) $ and $\widehat{E}_V(\mu,t) = \sum_k \widehat{\P}_V(\err^k_{\s_k}).$
    \item Let $(\mu^*, t^*) = \argmax_{\mathbf{M} \times \mathbf{T}} \widehat{C}_V(\mu,t)$ subject to $\widehat{E}_V(\mu,t) \le \varepsilon$. Return $(\theta(\mu^*), \{w_k(\mu^*)\}, t^*)$.
\end{itemize}

\section{Experiments}\label{sec:exp}
%As pursued in the rest of the text, we focus on binary classification in the regime of low errors.
\subsection{Experimental Setup and Baselines} \label{sec:setup_baselines}

%\textcolor{red}{Two suggestions - further training of the last layers in the setting of Tables 2 and 3; Second, train the backbone at a slower scale than the last layer, warm starting from the CE solution}

\textbf{Datasets and Model Class} We evaluate all methods on three benchmark vision tasks: CIFAR-10 \cite{CIFAR}, SVHN-10 \cite{SVHN} ($10$ classes), and Cats \& Dogs\footnote{\url{https://www.kaggle.com/c/dogs-vs-cats}} (binary). All models implemented below are DNNs with the RESNET-32 architecture \cite{resnet}, which is a standard model class in vision tasks. 20\% of the training data is reserved for validation in each dataset. All models are implemented in the tensorflow framework. The samples sizes and the best standard classification performance is presented in Table \ref{table:data+stdCl}. 
\begin{table}[h]
\tabcolsep=0.05cm
    \centering
     \begin{tabular}{c|  c c c  c } 
     %\hline
      \multirow{2}{*}{ Dataset}  & \multicolumn{3}{c}{Num. of Samples}{}  & \multirow{2}{*}{Std.~Error} \\ 
      & Train. & Test & Val. &  \\ 
     \hline
     \multirow{2}{*}{CIFAR-10 } & \multirow{2}{*}{45K} & \multirow{2}{*}{10K}  & \multirow{2}{*}{5K} & \multirow{2}{*}{9.58\%} \\ & & & & \\ 
     \multirow{2}{*}{SVHN-10 } & \multirow{2}{*}{65.9K} & \multirow{2}{*}{26K}  & \multirow{2}{*}{7.3K} & \multirow{2}{*}{3.86\%} \\ & & & &\\ 
     \multirow{2}{*}{Cats \& Dogs } & \multirow{2}{*}{18K} & \multirow{2}{*}{5K}  & \multirow{2}{*}{2K} & \multirow{2}{*}{5.72\%} \\ & & & &\\ 
    \end{tabular}
    \caption{Dataset sizes and standard classification error}
    \label{table:data+stdCl}
\end{table}

\textbf{Baselines:} We benchmark against three state of the art methods. The `selective net' and `deep gamblers' methods also require hyperparameter and threshold tuning as in our setup, and we do this in a brute force way on validation data, as in ours.

\emph{Softmax Response Thresholding} (SR) involves training a neural network for standard classification, and then thresholding its soft output to decide to reject. More formally, the decision is to reject if $\{ \mathrm{softmax}(f_1,\dots, f_K) < t\},$ where $f$ is the soft output, and $t$ is tuned on validation data. This simple scheme is known to have near-SOTA performance \cite{geifman2017selective, geifman2019selectivenet}.

\emph{Selective Net} (SN) is a DNN meta-architecture for SC \cite{geifman2019selectivenet}. The network provides three soft outputs - $(f,\gamma, \pi)$, where $f$ is an auxiliary classifier used to aid featurisation during training, and $\gamma, \pi$ is a gate-predictor pair. Selective net prescribes a loss function that trades off coverage and error via a multiplier $c$, and by fine-tuning a threshold on $\gamma$ to reject. We use the publicly available code\footnote{\url{https://github.com/geifmany/selectivenet}} to implement this, and a comprehensive sweep over the coverage and threshold hyper-parameters. We use $40$ valued grid for the parameter $c$ (with 10 equally spaced values in the range $[0.0, 0.65)$ and remaining $30$ values in the range $[0.65, 1.0]$). For the gating threshold $\gamma$, we use $100$ thresholds equally spaced in the range $[0,1]$, the same as for our scheme. 
    
% https://github.com/Z-T-WANG/NIPS2019DeepGamblers/
\emph{Deep Gamblers} (DG) is a method based on a novel loss function for SC  within the gating framework \cite{deep_gamblers}. The NNs have $K+1$ outputs - $f_1,\dots, f_K, f_?$. The cross-entropy loss is modified to $\sum \log\left(( f_{y_i}(x_i) + \smallo^{-1} f_?(x_i)\right),$ where $\smallo \in [1,K)$ is a hyperparameter that trades-off coverage and accuracy. Hard decisions are obtained by tuning the threshold of $f_?$ on a validation set. We adapt the public torch code\footnote{\url{https://github.com/Z-T-WANG/NIPS2019DeepGamblers/}} for this method to the Tensorflow framework. We used $40$ values of $\smallo$ spaced equally in the range {$[1,2)$}\footnote{We initially made a mistake and scanned $\smallo$ in $[1,2)$ instead of $[1,10)$. We then redid the experiment. with 40 values in $[1,10)$, and found that performance deteriorated. This is because the optimal $\smallo$ for these datasets lies in $[1,2),$ and the wider grid leads to a less refined search in this domain. Thus, values from the original experiment are reported. See Tables \ref{table:fixed_multi_class_results_new_table}, \ref{table:fixed_multi_class_results_target_cov_new_table} in \S\ref{appx:expts} for the values with a scan over $[1,10)$.}, and 100 values of thresholds in $[0,1]$. 

\subsection{Training One-Sided Classifiers}

\textbf{Loss Function} We use the loss function $\widetilde{M}^{\res}$ developed in \S\ref{sec:prac_meta}. In particular for $\widetilde{L}^{\res}_k,$ we use $\ell(z) = -\log(z),$ and for $\widetilde{C}_k, \ell'(z) = -\log(1-z).$

\textbf{Training of Backbones} As previously discussed, our models share a common backbone and have a separate output node for each OSC problem. We intialise this backbone with a base network trained using the cross-entropy loss (i.e.~a `warm start'). Note that this typically yields a strong featurisation for the data, and exploiting this structure requires us to not move too far away from the same. At the same time, due to the changed objective, it is necessary to at least adapt the final layer significantly. We attain this via a two-timescale procedure: the loss is set to the OSP Lagrangian, and the backbone is trained at a \emph{slower rate} than the last layer. Concretely, the last layer is updated at every epoch, while the backbone is updated every 20 epochs. This stabilises the backbone, while still adapting it to the particular OSP problem that the network is now trying to solve.

\textbf{Hyper-parameters}. All of the methods were trained using the train split and the model selection was performed on the validation set. The results are reported on the separate test data (which is standard for all three of the models considered). The minimax program on the Lagrangian was optimised using a two-timescale stochastic gradient descent-ascent, following the recent literature on nonconvex-concave minimax problems \cite{lin2019gradient}. In particular, we used Adam optimizer for training with initial learning rates of $(10^{-3}, 10^{-5})$ for the min and the max problems respectively for CIFAR-10 and SVHN-10, and of $(10^{-3}, 10^{-4})$ for Cats \& Dogs.\footnote{These rates were selected as follows: the standard classifier was trained with the rate $10^{-4}$, which is a typical value in vision tasks. We then picked one value of $\mu$, and trained models using rates in $(10^{-k}, 10^{-j})$ for $(j,k) \in [2:6] \times [2:6]$, tuned thresholds for models at $0.5\%$ target accuracy using validation data, and chose the pair that yielded the best validation coverage. Performance tended to be similar as long as $j \neq k$, and curiously, we found it slightly better to use a smaller rate for the max problem, which goes against the suggestions of Lin et al. \cite{lin2019gradient}.} These initial rates were reduced by a factor of $10$ after $50$ epochs, and training was run for $200$ epochs. The batch size was set to $128.$ 

We searched over 30 values of $\mu$ for each of our experiments - $10$ values equally spaced in $[0.01,1],$ and remaining 20 equally spaced in $[1,16]$. We further used 100 values of thresholds equally spaced in $[0,1]$. 

\subsection{Results} The key takeaway of our empirical results is the significant increase in performance of our SC scheme when compared to the baselines. We also include some observations about the structure of the solutions obtained.

\subsubsection{Performance}

Note - the Tables cited below all appear on page 13.

\textbf{Performance at Low Target Error} is presented in Table \ref{table:multi_class_results}, which reports coverage at three (small) targeted values of error - $\nicefrac{1}{2},1,$ and $2$ percent - that are in line with the low target error regime that is the main focus of the paper. Notice that these target error values are far below the best error obtained for standard classification (Table \ref{table:data+stdCl}). We observe that the performance of our SC methods is significantly higher than the SOTA methods, especially in the case of CIFAR-10 and Cats \& Dogs, where we gain over $4\%$ in coverage at the $0.5\%$ design error. The effect is weaker in SVHN, which we suspect is due to saturation of performance in this simpler dataset. 

\textbf{Performance at High Target Coverage} is presented in Table \ref{table:multi_class_results_target_cov}. This refers to the coverage constrained SC formulation discussed in \S\ref{sec:alternate_goals}. For these experiments, we use the same $\mu$ values (to avoid retraining), but choose thresholds such that the coverage of the resulting model exceeds the stated target, and the models with the lowest error at this threshold are chosen. We observe that at target coverage $100\%,$ the SR solution outperfoms all others. This is expected, since 100\% coverage corresponds to standard classification, and the SR objective is tuned to this, while the others are not. Surprisingly, for coverage below 100\%, our OSP-based relaxations deliver stronger performance than the benchmarks. Note that this is not due to the low target error performance, because (besides SVHN), the errors attained at these coverage are signficantly above the low target errors investigated in Table \ref{table:multi_class_results}. This shows that our formulation is also effective in the high-coverage regime. 

\textbf{Coverage-Error Curves} for the CIFAR-10 dataset are shown in Fig.~\ref{fig:cov-err_curve}. These curves plot the best coverage obtained by training at a given target error level using each of the methods discussed.\footnote{In particular, we train models at target errors $\varepsilon_i = (i/2)\%$ for $i \in [1:20]$. We then obtain the achieved test error rates $\widehat{\varepsilon}_i$ and coverages $c_i$ for these models. The curves linearly interpolate between $(\widehat{\varepsilon}_i, c_i)$ and $(\widehat{\varepsilon}_{i+1}, c_{i+1})$.} We find that the coverage obtained by our method uniformly outperform DG and SN, and also outperform SR for the bulk of target errors, except those very close to the best standard error attainable.  This illustrates that our scheme is effective across target error levels. We find this rather surprising since we designed our method with explicit focus on the low target error regime. Tables \ref{table:multi_class_results} and \ref{table:multi_class_results_target_cov} can be seen as detailed looks at the left (error $< 2$) and the upper (coverage $>90$) ends of these curves. 

\textbf{Observations regarding baselines}. Tables \ref{table:multi_class_results},\ref{table:multi_class_results_target_cov}, and Figure \ref{fig:cov-err_curve} all show that across regimes, DG and SN perform similarly to SR, and are frequently beaten by it. This observation is essentially consistent with the results presented in previous work \cite{geifman2019selectivenet, deep_gamblers}, and supports our earlier claims that the prior SOTA methods for selective classification do not meaningfully improve on na\"{i}ve methods. To alleviate concerns about implementation, we emphasise that we performed a comprehensive hyperparameter search for both SN and DG, and the only change is to use RESNETs instead of VGG.

\begin{figure}[t]
\begin{center}%\vspace{-10pt}
   \includegraphics[width=.6\linewidth]{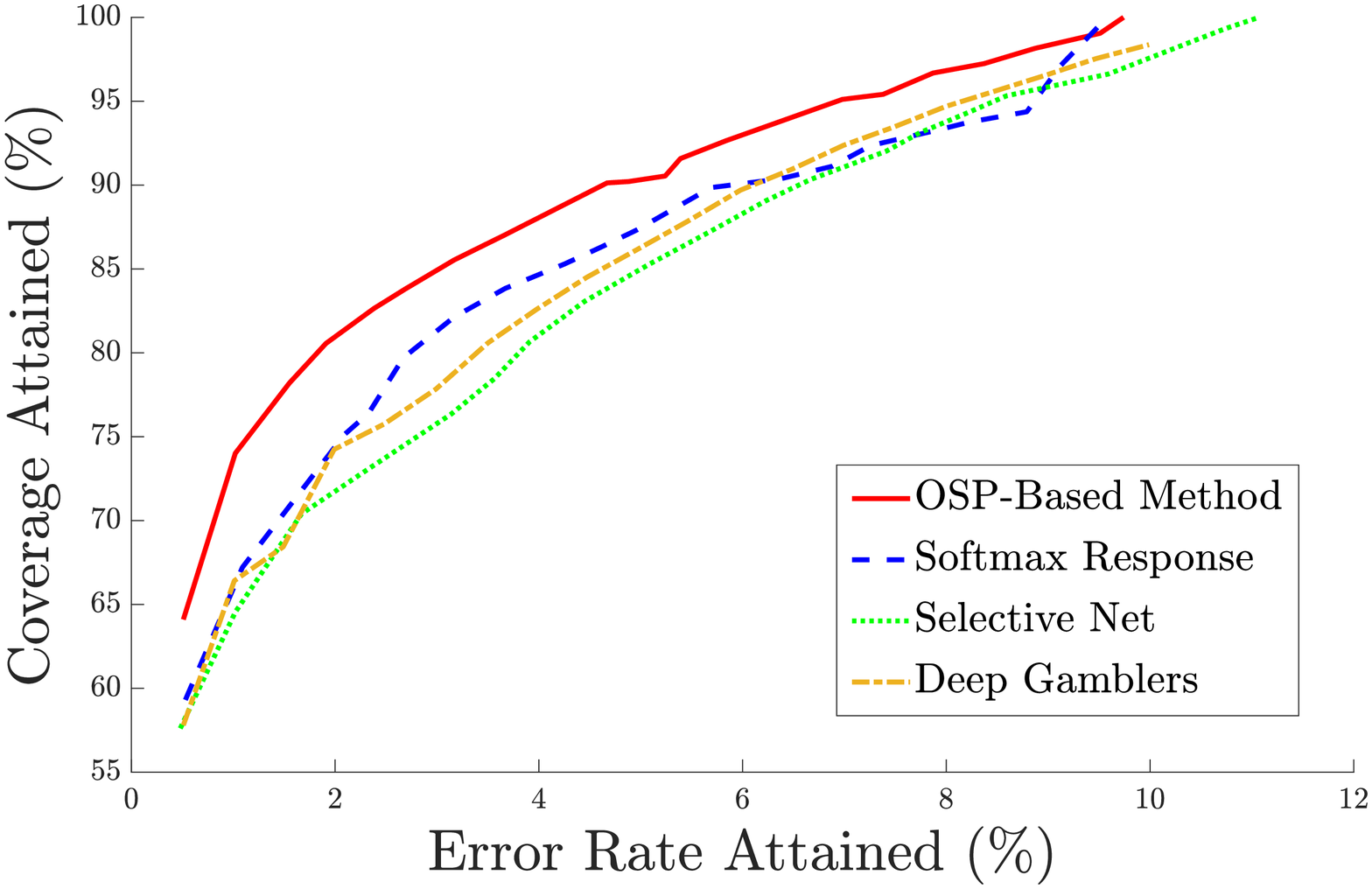}
\end{center}%\vspace{-10pt}
   \caption{Coverage vs Error Curves for the CIFAR-10 dataset. Higher values of coverage are better. Notice the curious behaviour of SR in that the curve's slope sharply changes close to the best standard error rate.}
\label{fig:cov-err_curve}
\end{figure}

\subsubsection{Structure of the Solutions}

\textbf{Overlap of OSP solutions is small}. Table \ref{tab:percentage-overlap-between-two-one-sided-learners} shows the probability mass of the ambiguous regions for our raw OSP solutions (i.e., the raw sets $\{x: f_k > t\}$ without the max-assignment $\s_k =\{x: f_k > t\} \cap \{ x: k = \argmax f_k\}$) for the models of Table \ref{table:multi_class_results}. We find that this overlap is very small - much smaller than the $2\varepsilon$ bound in Prop.~\ref{prop:OSP_not_lossy}. Empirically, these sets are essentially disjoint, and so the training process is close to tight for the SC problem. We believe that this effect is mainly due to the simple tuning enabled by the softmax normalisation of OSP problem outputs described in \S\ref{sec:prac_meta}.

\textbf{Consistency of rejection regions} We say that a sequence of models trained at error levels $\varepsilon_i$ have consistent rejection regions if for every $\varepsilon_i < \varepsilon_j$, if $\rr_i, \rr_j$ are the rejection regions for models trained at these errors, then $\P(\rr_i \cap \rr_j^c)$ is very small. This means that points that are rejected when designing at a higher error level continue to be rejected for stricter error control. Such consistency may be useful for building cascades of models, or for using the error level at which a point is rejected as a measure of uncertainty.

We found that the models obtained by our procedure are remarkably consistent in the high-accuracy regime. Concretely, for $\varepsilon_i  = (i/2)\%$ for $\in [1:5],$ for both CIFAR-10 and Cats \& Dogs test sets, the models were entirely consistent, i.e. $\rr_j \subset \rr_i$ for $j > i$\footnote{Due to time constraints we only checked for higher values of $i$ in the CIFAR-10 case, in which the trend continued until $i = 20$, that is, until full coverage. A curious observation in this case was that in all $20$ models for the CIFAR-10 dataset, the same value of $\mu$ was best, and the models differed only in the thresholds (this did not occur for SVHN and Cats v/s Dogs). While this obviously implies consistency of the rejection regions, it is unexpected, and suggests that there may be room to improve in our training methodology.}, while for SVHN, the only violation was that $|\rr_{2.5\%} \cap \rr_{2\%}^c| = 2$. Since the test dataset for SVHN has size $>7000,$ this is a tiny empirical probability of inconsistency of $< 0.03\%$. 

\section{Conclusion}

We have proposed a new formulation for selective classification, which leads to a novel one-sided prediction method. The formulation is naturally motivated, and is equivalent to other formulations appearing in the literature. It is also amenable to standard statistical analyses. The resulting method is flexible, efficiently trainable via standard techniques, and further outperforms state of the art methods across target error regimes. Further, it is the first method to non-trivially outperform na\"{i}ve post-hoc solutions, and thus, in our opinion, represents a significant step in the practical approaches to selective classification.

\begin{table*}[ht]
\tabcolsep=0.2cm %\vspace{-10pt}
    \centering {%\resizebox{2\columnwidth}{!}{%
     \begin{tabular}{ c | c |  c c  c c  c c  c c } 
     \multirow{2}{*}{Dataset} & {Target} & \multicolumn{2}{c}{OSP-based}{}   & \multicolumn{2}{c}{SR} & \multicolumn{2}{c}{SN}  & \multicolumn{2}{c}{DG} \\ 
      & Error & Cov. & Error & Cov. & Error  & Cov. & Error & Cov. & Error \\ 
     \hline\hline
     %\multirow{3}{*}{CIFAR-10}     & 10\% & 100 & 9.74  & 98.52 & 10.58 & 100 & 11.07 & 97.65 & 9.97 \\
    \multirow{3}{*}{CIFAR-10}       & 2\%  & \textbf{80.6} & 1.91  & 75.1 & 2.09 & 73.0 & 2.31 & 74.2 & 1.98 \\
                    & 1\%  & \textbf{74.0} & 1.02  & 67.2 & 1.09 & 64.5 & 1.02  & 66.4 & 1.01\\
                    & 0.5\%  & \textbf{64.1} & 0.51  &  59.3 & 0.53 & 57.6 & 0.48 & 57.8 & 0.51 \\
     \hline
    % \multirow{3}{*}{SVHN-10}     & 10\%  & 100 & 4.15 & 99.97 & 3.86 & 100  & 4.59 & 100 & 4.38 \\
    \multirow{3}{*}{SVHN-10}      & 2\%  & \textbf{95.8} & 1.99 & 95.7 & 2.06 & 93.5  & 2.03 & 94.8 & 1.99\\
                    & 1\%  & \textbf{90.1} & 1.03  & 88.4 & 0.99 & 86.5 & 1.04 & 89.5 & 1.01\\
                    & 0.5\%  & \textbf{82.4} & 0.51  & 77.3 & 0.51 & 79.2 & 0.51 & 81.6 & 0.49 \\
    \hline
     %\multirow{3}{*}{Cats \& Dogs}  & 2\%  & \textbf{91.5} & 1.96  & 88.2 & 2.03 & 84.3 & 1.94 & 87.4 & 1.94 \\ 
     %              & 1\%  & \textbf{85.1} & 0.98  & 80.2 & 0.97 & 78.0 & 0.98 &  81.7 & 0.98\\ 
     %              & 0.5\% & \textbf{79.7} & 0.49  & 73.2 & 0.49 & 70.5 & 0.46 & 74.5 &  0.48\\
     \multirow{3}{*}{Cats \& Dogs}  & 2\%  & \textbf{90.5} & 1.98  & 88.2 & 2.03 & 84.3 & 1.94 & 87.4 & 1.94 \\ 
                   & 1\%  & \textbf{85.4} & 0.98  & 80.2 & 0.97 & 78.0 & 0.98 &  81.7 & 0.98\\ 
                   & 0.5\% & \textbf{78.7} & 0.49  & 73.2 & 0.49 & 70.5 & 0.46 & 74.5 &  0.48\\
    \end{tabular}}\vspace{-7pt}
    \caption{Performance at Low Target Error. The OSP-based scheme is our proposal. SR, SN, DG correspond to softmax-response, selective net, deep gamblers. Errors are rounded to two decimals, and coverage to one.}\vspace{-5pt}
    \label{table:multi_class_results}
\end{table*}

\begin{table*}[ht]
\tabcolsep=0.2cm
    \centering {%\resizebox{2\columnwidth}{!}{%
     \begin{tabular}{ c | c |  c c  c c  c c  c c } 
     \multirow{2}{*}{Dataset} & {Target} & \multicolumn{2}{c}{OSP-based}{}   & \multicolumn{2}{c}{SR} & \multicolumn{2}{c}{SN}  & \multicolumn{2}{c}{DG} \\ 
      & Coverage & Cov. & Error & Cov. & Error  & Cov. & Error & Cov. & Error \\ 
     \hline\hline
     \multirow{3}{*}{CIFAR-10}     
        & 100\%  & 100 & { 9.74 } & 99.99 & \textbf{9.58} & 100 & 11.07 & 100 & 10.81 \\
%        & 97.5\% & 97.53 & 8.29 & 97.49 & 9.96 & 97.65 & 9.96 & 97.54 & 9.47 \\
%        & 97\%   & 97.11 & 7.91 & 96.9 & 9.65 & 96.75 & 9.51 & 97.03 & 9.19 \\
        & 95\%   & 95.1 & \textbf{6.98} & 95.2 & 8.74 & 94.7 & 8.34 & 95.1 & 8.21 \\
        & 90\%   & 90.0 & \textbf{4.67} & 90.5 & 6.52 & 89.6 & 6.45 & 90.1 & 6.14 \\
     \hline
     \multirow{3}{*}{SVHN-10}     
        & 100\% & 100 & 4.27 & 99.97 & \textbf{3.86} & 100 & 4.27 & 100 & 4.03  \\
%        & 97.5\% & 97.61 & 3.02 & 97.53 & 2.68 & 97.67 & 3.31 & 97.51 & 2.85 \\
%        & 97\% & 97.03 & 2.61 & 97.08 & 2.56 & 97.13 & 3.14 & 97.02 & 2.64 \\
        & 95\% & 95.1 & \textbf{1.83} & 95.1 & 1.86 & 95.1 & 2.53 & 95.0 & 2.05 \\
        & 90\% & 90.1 & \textbf{1.01} & 90.0 & 1.04 & 90.1 & 1.31 & 90.0 & 1.06 \\
     \hline
     \multirow{3}{*}{Cats \& Dogs}     
        & 100\% & 100 & 5.93 & 100 & \textbf{5.72} & 100 & 7.36 & 100 & 6.16  \\
%        & 97.5\% &  &  & 97.52 & 4.38 & 97.96 & 6.3 & 97.72 & 5.36 \\
%        & 97\% &  &  & 97.01 & 4.24 & 97.38 & 6.14 & 97.26 & 5.2 \\
        & 95\% & 95.1 & \textbf{ 2.97 } & 95.0 & 3.46 & 95.2  & 5.1 & 95.1 & 4.28  \\
        & 90\% & 90.0 & \textbf{ 1.74 } & 90.0 & 2.28 & 90.2 & 3.3 & 90.0 & 2.50  \\
    \end{tabular}}\vspace{-6pt}
    \caption{Performance at High Target Coverage. Same notation as Table \ref{table:multi_class_results}.} \vspace{-5pt}
    \label{table:multi_class_results_target_cov}
\end{table*} 

\begin{table}[ht]\tabcolsep=0.2cm
    \centering 
     \begin{tabular}{ c | c | c } 
     Dataset & Target Error &  Overlap  \\ 
     \hline\hline
     \multirow{3}{*}{CIFAR-10} & 2\% & 0.09\% \\ 
                      & 1\% & 0.01\% \\ 
                      & 0.5\% & 0.00\% \\ 
     \hline
     \multirow{3}{*}{SVHN-10} & 2\% &  0.05\%\\ 
                      & 1\% & 0.01\% \\ 
                      & 0.5\% & 0.00\% \\ 
     \hline
     \multirow{3}{*}{Cats \& Dogs}  & 2\% &  0.07\% \\ 
                      & 1\% & 0.01\% \\ 
                      & 0.5\% &  0.00\%
    \end{tabular}\vspace{-6pt}
    \caption{Size of overlap between OSP sets in Table \ref{table:multi_class_results} }%\vspace{-5pt}
    \label{tab:percentage-overlap-between-two-one-sided-learners}
\end{table}

% \begin{figure}[hbt]
% \begin{center}\vspace{-10pt}
%   \includegraphics[width=.99\linewidth]{Cov-Err.eps}
% \end{center}\vspace{-10pt}
%   \caption{Coverage vs Error Curves for the CIFAR-10 dataset. Higher values of coverage are better. Notice the curious behaviour of SR in that the curve's slope sharply changes close to the best standard error rate.}
% \label{fig:cov-err_curve}\vspace{-10pt}
% \end{figure}

% \section{Conclusion}

% We have presented a novel method for selective classification that was motivated in principled manner by relaxing a new formulation of the problem. The formulation is natural, equivalent to prior proposals, and is amenable to standard statistical analyses. The OSP-based relaxation is theoretically efficient in the low target error regime, and the resulting method is efficiently trainable via standard techniques, and outperforms SOTA methods across target error levels. Further, it is the first method to non-trivially outperform na\"{i}ve post-hoc solutions, and thus represents a significant step in the practical approaches to selective classification.
\newpage

\subsection*{Acknowledgements}

This research was supported by National Science Foundation grants CCF-2007350 (VS), CCF-2022446(VS), CCF-1955981 (VS), the Data Science Faculty Fellowship from the Rafik B. Hariri Institute, the Office of Naval Research Grant N0014-18-1-2257 and by a gift from the ARM corporation.

%durmus alp emre acar

We would like to thank Durmus Alp Emre Acar for helpful discussions regarding implementation of the methods.

\printbibliography

\clearpage

\newpage
\appendix
\thispagestyle{empty}

\begin{center}
    {\Large \textbf{\underline{Appendices to Selective Classification via One-Sided Prediction}} }
\end{center}

\section{Appendix to \S\ref{sec:form} }

\subsection{Proof of Proposition \ref{prop:OSP_not_lossy}}\label{appx:prop1pf}

\begin{proof}

We recall the notation. $\alpha \in [0,1]^K$ is such that $\sum \alpha_k \le 1.$ The $\tcal^{\alpha}_k$ are the optimising solutions to the OSP problems at error $\alpha_k \varepsilon,$ i.e.  \[ \tcal_k^{\alpha} \in  \argmax\, \P(\tcal) \textrm{ s.t. } \P(X \in \tcal, Y \neq k) \le \alpha_k \varepsilon,\] while the $\s_k^{\alpha}$ are produced by removing the smaller overlap with smaller labels in $\tcal_k^{\alpha},$ i.e. \[ \s_k^{\alpha} = \tcal_k^{\alpha} \setminus \bigcup_{k' < k} \tcal_{k'}^\alpha.\]

We first argue that the total overlap of the $\tcal$s is small. \begin{mylem}\label{lem:overlap_is_small}
Let $\tcal_k^{\alpha}$ be generated as above. Then \[ \sum_k \P( \bigcup_{ k'\neq k} \tcal_k^{\alpha} \cap \tcal_{k'}^{\alpha} ) \le 2\varepsilon. \] Since the total overlap is $\bigcup_{k} \left(\tcal_k^{\alpha} \cap \bigcup_{k' \neq k} \tcal_{k'}^{\alpha}\right)$, this also controls the probability of the total overlap, that is, \[ \P( \bigcup_{k, k'\neq k} \tcal_k^\alpha \cap \tcal_{k'}^{\alpha}) \le 2\varepsilon.\]
\end{mylem}

This lemma is sufficient to show the claim,  since \begin{align*} \sum_k \P(\s_k) &=\sum_k \P(\tcal_k^\alpha \setminus \bigcup_{k' < k} \tcal_{k'}^\alpha) \\
                                                                                  &\ge \sum_k \P(\tcal_k^{\alpha}) - \sum_k \P(\tcal_k^\alpha \cap \bigcup_{k' < k} \tcal_{k'}^{\alpha}) \\
                                                                                  &\ge \sum_k \P(\tcal_k^{\alpha}) - \sum_k \P(\tcal_k^\alpha \cap \bigcup_{k' \neq k} \tcal_{k'}^{\alpha}) \\
                                                                                  &\ge C(\varepsilon;\sss) - \sum_k \P(\tcal_k^\alpha \cap \bigcup_{k' \neq k} \tcal_{k'}^{\alpha})\\
                                                                                  &\ge C(\varepsilon;\sss) - 2\varepsilon,\end{align*}
where we have used that $\sum_k \P(\tcal_k^{\alpha}) \ge C(\varepsilon;\sss),$ which holds because the $\tcal^{\alpha}_k$ optimise a relaxation of (\ref{SC}), and the final inequality is due to the above lemma.\end{proof}

We conclude by proving the above lemma
\begin{proof}[Proof of Lemma \ref{lem:overlap_is_small}]
Since the labels of $Y$ are mutually, exclusive, 
\[ \sum_k \P( \bigcup_{k'\neq k} \tcal_k^{\alpha} \cap \tcal_{k'}^{\alpha} ) = \sum_k  \sum_j \P( \bigcup_{k'\neq k} \tcal_k^{\alpha} \cap \tcal_{k'}^{\alpha} , Y = j). \]

Applying Fubini's theorem, and recalling that the probability of an intersection of events is smaller than the probability of either of the events, we see that \begin{align*}
    \sum_k \P( \bigcup_{k'\neq k} \tcal_k^{\alpha} \cap \tcal_{k'}^{\alpha} ) &=  \sum_k \sum_j  \P( \tcal_k^{\alpha} \cap (\bigcup_{k' \neq k} \tcal_{k'}^{\alpha}), Y = j ) \\
                                                                          &\le \sum_k \left( \sum_{j \neq k} \P(\tcal_k^{\alpha}, Y = j)\right) + \P(\bigcup_{k' \neq k} \tcal_{k'}^\alpha, Y = k),
\end{align*}
Now, notice that the sum in the brackets is simply $\P(\tcal_k^{\alpha}, Y \neq k)$. Taking the union bound over the second probability, we find the upper bound \begin{align*}
    \sum_k \P( \bigcup_{k'\neq k} \tcal_k^{\alpha} \cap \tcal_{k'}^{\alpha} ) &\le \sum_k \P(\tcal_k^{\alpha}, Y \neq k) + \sum_k \sum_{k' \neq k} \P(\tcal_{k'}^{\alpha}, Y = k) \\
                                                                          &= \sum_k \P(\tcal_k^{\alpha}, Y \neq k) + \sum_{k'} \sum_{k \neq k'} \P(\tcal_{k'}^{\alpha}, Y = k) \\
                                                                          &= \sum_k \P(\tcal_k^{\alpha}, Y \neq k) + \sum_{k'} \P(\tcal_{k'}^{\alpha}, Y \neq k') \\
                                                                          &= 2 \sum_k \P(\tcal_k^{\alpha}, Y \neq k)\\
                                                                          &\le 2 \sum \alpha_k \varepsilon = 2\varepsilon,
\end{align*}
where the first equality is by Fubini's theorem again, the second equality is by the disjointness of the values of $Y$, and the final inequality is due to the constraints of the OSP problems.

\end{proof}

\subsection{Asyptotically Feasible Finite Sample Analysis for SC}\label{appx:SC_gen}

In parallel to the OSP problems, one can directly give finite sample analyses for the SC problem. We begin by defining the solution concept here.

\begin{defi}
We say that a class $\sss$ is learnable with abstention if for every $(\delta, \zeta, \sigma, \nu) \in (0,1)^4,$ there exists a finite $m(\delta,\sigma, \nu, \zeta)$ and an algorithm $\mathfrak{A}: (\mathcal{X} \times \{+,-\})^m \to \sss^K$ such that for any law $\P,$ and $\varepsilon > 0,$ given $n$ i.i.d.~samples drawn from $\P$, the algorithm produces sets $\{\s_k\}$ from $\sss$ such that with probability at least $1-\delta$ over the data, \begin{align*} \sum_k\P(\s_k) &\ge C(\varepsilon;\sss) - \sigma \\ \P(\err_{\{\s_k\}}) &\le \varepsilon + \nu \\ \P(\bigcup_{k, k'\neq k}\s_k \cap \s_{k'}) &\le \zeta.\end{align*}
\end{defi}

Notice that the recovered sets need not be disjoint, which may be amended by by eliminating the overlap from one of the sets as in \S\ref{sec:OSP}. The resulting (improper) sets attain coverage of at least $C - \zeta - \nu$ with high probability. 

The main point characterisation here is similar, \begin{myprop}\label{prop:SC_vc}
A class $\sss$ is learnable with abstention if and only if it has finite VC dimension, and further, \[ m( \delta, \sigma, \nu,\zeta) = \widetilde{O}( \mathrm{poly}(K) \max( \sigma^{-2}, \nu^{-2}, \zeta^{-1}) \textsc{vc}(\sss)). \]
\end{myprop}

The proof of the necessity of finite VC dimension follows from observing that if the data is realisable, i.e., corresponds to $Y = 2\indi\{ X \in \s\} - 1$ for some $\s \in \sss$ then at least one of the recovered sets is a good classifier, at which point standard lower bounds for realisable PAC learning apply. The sufficiency follows from utilising the finite VC property to uniformly bound errors incurred by empirical means. The proof is presented in \S\ref{appx:samplecomp_pfs}.

\subsection{Proofs of Propositions \ref{prop:OSP_vc} and \ref{prop:SC_vc}}\label{appx:samplecomp_pfs}

\subsubsection{Proofs of Necessity of Finite VC dimension}

In both cases, we reduce the problems to realisable PAC learning, and invoke standard bounds for the same, for instance the one of Chapter 3 in the book by Mohri et al. \cite[Ch.3]{mohribook}. To this end, suppose $\delta \le 1/100$, and consider the restricted class of joint laws $\P$ such that $\P(Y = k| X = x) = \mathds{1}\{X \in \s_{k,*}\}$ for some disjoint $\{\s_{k,*}\} \in \sss$ that together cover $\mathcal{X}$.\footnote{Strictly speaking, this requires that $\sss$ is rich enough to express such a class. This is a very mild assumption. For the purposes of the lower bound, in fact, this can be weakened still - all we really need is a binary law, and that if $\s \in \sss,$ then $\s^c \in \sss$. Then we can take $\P(Y = 1|X = x) = \mathds{1}\{X \in \s\}, \P(Y = 2|X = x) = \mathds{1}\{X \in \s^c\}$, and the entirety of the following argument goes through without change.}

\begin{proof}[Proof for {One-Sided Prediction}] Notice that $\s_1^*$ is feasible for OSP-1 for any value of $\varepsilon.$ If we can solve OSP-1, then we would have found a set $\s$ such that \begin{align*} \P(\s) \ge \P(\s_{1,*})& - \sigma \\ \P(X \in \s \cap \s_{1,*}^c)  &= \P(X \in \s, Y = 2) \le \varepsilon + \nu. \end{align*}

Further, \begin{align*}
    \P(\s^c) & = \P(\s^c \cap \s_{1,*}) + \P(\s^c \cap \s_{1,*}^c) \\
             &= \P(\s^c \cap \s_{1,*}) + \P(\s_{1,*}^c) - \P(\s \cap \s_{1,*}^c).
\end{align*} 

But $\P(\s^c) = 1-\P(\s) \le 1 - \P(\s_{1,*}) + \sigma \le \P(\s_{1,*}^c) + \sigma.$

Thus, we have \begin{align*}
        \P(\s^c \cap \s_{1,*}) + \P(\s_{1,*}^c) - \P(\s \cap \s_{1,*}^c) & \le \P(\s_{1,*}^c) + \sigma \\
        \implies \P(\s^c \cap \s_{1,*}) &\le \sigma + \P(\s \cap \s_{1,*}^c) \le \varepsilon + \sigma + \nu.
\end{align*}

But then, viewed as a standard classifier for the problem separating the class $\{1\}$ from $[2:K],$ $\s$ has risk at most $2\varepsilon + \sigma + \nu$. Consequently, an algorithm for solving OSP yields an algorithm for realisable PAC learning for this problem. Thus, invoking the appropriate standard lower bound, we conclude that \[ m_{\mathrm{OSP}} \ge \frac{\textsc{vc}(\sss) - 1}{32 (2\varepsilon + \sigma + \nu)}. \qedhere\] \end{proof}

\begin{proof}[Proof for {Learning With Abstention}] Notice that $\{\s_{k,*}\}$ serve as a feasible solution for any $\varepsilon$, and have total coverage $1$. Thus, if SC is possible, we may recover sets $\{\s_k\}$ such that \begin{align*}
    \sum \P(\s_k) &\ge 1 - \sigma \\
    \P(\err_{\{\s_k\}}) &\le \varepsilon + \nu \\
    \P\left(\bigcup_k ( \s_k \cap \bigcup_{k'\neq k} \s_{k'})\right) &\le \nu.
\end{align*}

Now notice that $\s_{1,*}$ and $\s_{1,*}^c$ correspond to the realisable classifiers for the binary classification problem separating $\{1\}$ from $[2:K]$.\footnote{Again, this needs that $\sss$ is rich enough to include $\s_{1,*}^c$.} But, in the same way, we may view $\s_1$ and $\s_1^c$ as binary classifiers for this problem. Now notice that for this binary classification problem, $\s_1$ incurs small error. Indeed, denoting $\s_{\neq 1} = \bigcup_{k' \neq 1} \s_{k'},$ we find that 

 \begin{align*}
    \P(X \in \s_1, Y \neq 1) + \P(X \in \s_1^c, Y = 1) &= \P(X \in \s_1, Y \neq 1) + \P(X \in  \s_{\neq 1} \cap \s_1^c, Y = 1) \\ &\quad + \P(X \in \s_{\neq 1}^c \cap \s_1^c, Y = 1)\\
                                                    &\le \P(X \in \s_1, Y \neq 1) + \P(X \in  \s_{\neq 1}, Y = 1) \\ &\quad + \P(X \in \s_1^c \cap \s_{\neq 1}^c)\\                
                                                    &\le \P(\err_{\{\s_k\}}) + (1 - \P(\s_1 \cup \s_{\neq 1}^c)))\\
                                                    &\le \varepsilon + K\nu + \sigma + \zeta,. 
\end{align*}

where the second line's inequality is just non-negativity of probabilities, and the third line's inequality is due the fact that $\P(\err)$ is controlled, and the following inclusion-exclusion argument, first note that \begin{align*}
    \P(\s_1 \cup \s_{\neq 1}) &= \P(\bigcup \s_k) = \sum_k \P(\s_k) - \sum_k \P( \s_k \cap \bigcup_{k' > k} \s_{k'}).
\end{align*}

Next, observe that if $j > k,$ $\s_j \subset \bigcup_{k' > k} \s_{k'},$ and similarly $\bigcup_{k'>j}\s_{k'} \subset \bigcup_{k' > k} \s_k$. Thus, \[ \P(\s_1 \cup \s_{\neq 1}) = \P(\bigcup \s_k) \ge \sum_k \P(\s_k) - K \P( \s_1 \cap \bigcup_{k' > 1} \s_{k'}).\] Now invoking the SC solution conditions, the first sum is at least $1-\sigma,$ while the second probability is bounded by the probability of overlap, giving \[\P(\s_1 \cup \s_{\neq 1}) \ge 1 - \sigma - K\nu. \]

Thus, a SC yields a realisable PAC learner for the binary classifier problem separating $\{1\}$ from $[2:K],$ giving the bound \[ m_{\mathrm{SC}} \ge \frac{\textsc{vc}(\sss) - 1}{32 ( \varepsilon + \sigma + K\nu  + \zeta)}. \qedhere \]
\end{proof}

Note that these bounds are likely loose. The problems have plenty of structure that is not exploited in either of the above statements, and tighter inequalities would be of interest. However the point we intend to pursue - that assuming finiteness of VC dimensions in the upper bound analyses is not lossy, is sufficiently made above.

\subsubsection{Proofs of the Upper Bounds}

We mainly make use of the following uniform generalisation bound on the suprema of empirical processes due to the finiteness of VC dimension. This is, again, standard \cite{mohribook}.

\begin{mylem}\label{lem:vc_gen_bd}
Let $\sss$ have finite VC dimension. Then for any distribution $\P$, if $\widehat{P}_m$ denotes the empirical law induced by $m$ i.i.d.~samples from $\P$, then with probability at least $1-\delta$ over these samples, \[ \sup_{\s \in \sss, k \in [1:K]} | \widehat{\P}_m( X \in \s, Y = k) - \P( X \in \s, Y = k) | \le C_K\sqrt{ \frac{\textsc{vc}(\sss) \log m + \log(C/\delta)}{m}}, \] where $C_K$ is a constant independent of $\sss, \delta, \P, m.$
\end{mylem}

Notice that by summing over the values of $Y$, this also controls the error in the objects $\P( X \in \s)$ and $\P(X \in \s, Y \neq k)$, possibly with an error blowup of $K$, which can be absorbed into $C_K$.

For the purposes of the following, let $\Delta_{m, \sss}(\delta)$ be the value of the upper bound above.

\begin{proof}[Proof of Upper Bound for OSP]
    For $\alpha \in [0,1],$ define $\sss_\alpha \subset \sss$ to be the subset of $\s$s that have $\P(\err^1_\s) \le \alpha$, and let $\sigma, \nu$ be quantities that we will choose. 
    
    We give a two phase scheme - first we collect all sets $\s$ such that \( \widehat{P}_m (\err^1_{\s}) \le \varepsilon + \nu/2 \) into the set \( \widehat{\sss}_{\varepsilon + \nu/2}\). Notice that as long as $\nu/2 > \ddd{2},$ we have w.p. at least $1-\delta/2$ that \[ \sss_{\varepsilon} \subset \widehat{\sss}_{\varepsilon + \nu/2} \subset \sss_{\varepsilon + \nu}.\]
    
    Due to the upper inclusion, with probability at least $1-\delta/2$, every set in $\widehat{\sss}_{\varepsilon + \nu/2}$ has error level at most $\varepsilon + \nu.$
    
    Next, we choose the $\s \in \widehat{\sss}_{\varepsilon + \nu/2}$ that has the biggest coverage. If $\sss_{\varepsilon} \subset \widehat{\sss}_{\varepsilon + \nu/2},$ and $\sigma > \ddd{2},$ we are again assured that the selected answer will be at least $\sup_{\s \in \sss_{\varepsilon}} \P(\s) - \ddd{2} > L_k - \sigma$ with probability at least $1-\delta/2$. By the union bound, these will hold simultaneously with probability at least $1-\delta$. Since we want the smallest $\sigma, \nu,$ but for the arguments to follow we need that these are bigger than $2\ddd{2},$ we can set \[\nu = \sigma = 4\ddd{2} = 4C\sqrt{ \frac{\textsc{vc}(\sss) \log m + \log(2C/\delta)}{m}}, \] concluding the proof. \qedhere    
\end{proof}

\begin{proof}[Proof of Upper Bound for SC] This proceeds similarly to the above. To neatly present this, we let $\mathscr{R} = \{ \s : \s = \bigcup_{k, k' \neq k} \s_k \cap \s_{k'} , \{\s_k\}) \in \sss\}$ be the class of sets obtained by taking pairwise intersection of $k$-tuples in $\sss$. Note that VC dimesnsion of the sets obtained by taking pairwise intersection of sets in $\sss$ at most doubles the VC dimesnsion, while taking the $\binom{K}{2}$ unions in turn blows it up by a factor of $O(K^2 \log K$ by Lemma 3.2.3 of Blumer et al. \cite{blumer1989learnability}. Thus $\vc(\mathscr{R}) = O( K^2\vc(\sss) \log K)$. Now we may proceed as above, first by filtering the pairs of sets that satisfy the intersection constraint with value $\zeta/2$ on the empirical distribution, and then similarly checking the sum-error constraints and finally optimising the sum of their masses. The bounds are the same as the above, except with $\vc(\sss)$ replaced by $O(K^2\vc(\sss)\log K)$.    
\end{proof}

\subsubsection{Analyses not pursued here}

We first point out that there is nothing special about the VC theoretic analysis here - alternate methods like Rademacher complexity or a covering number analysis may replace Lemma \ref{lem:vc_gen_bd}. Similarly, the same analysis could be extended, via Rademacher complexities, to the setting of indicators relaxed to Lipschitz surrogates by exploiting Talagrand's lemma.

We note a few further analyses that we do not pursue here - firstly, using the technique of Rigollet \& Tong \cite{rigollet2011neyman}, it should be possible to give analyses for SC under convex surrogates of the indicator losses and a slight extension of the class $\sss$ while directly attaining the constraints (instead of asymptotically) with high prob. Additionally, a number of papers concentrate on deriving fast rates for the excess risks under the assumption of realizability (i.e., under the assumption that level sets of $\eta$ can be expressed via $\sss$), and that Tsybakov's noise condition holds at the level relevant to the optimal solution.

\newpage

\section{Algorithmic rewriting of Section \ref{sec:prac_meta} }\label{appx:algorithmics}

We specify the conclusions of \S\ref{sec:prac_meta} without any of the justifying development.

\noindent \textbf{Model class and Architecture} We use a DNN with the following structure:\begin{itemize}
    \item A `backbone', parametrised by $\theta$, which may have any convenient architecture.
    \item A `last layer' with $K$ outputs, denoted $f_k$, and associated weights $w_k$ for each. We denote $\mathbf{w} = (w_1, \dots, w_K)$.
    \item Let $\xi_\theta(x)$ denote the backbone's output on a point $x$. The DNN's outputs are \[ f(x; \theta, \mathbf{w}) = (f_1,\dots, f_k)(x;\theta, \mathbf{w}) = \mathrm{softmax}( \langle w_1, \xi_\theta(x)\rangle, \dots, \langle w_K, \xi_\theta(x)\rangle).\]
\end{itemize}

\noindent \textbf{Objective function and Training} We use the following objective function, where the $\{(x_i, y_i)\}_{i = 1}^n$ comprise the training dataset, $\theta, \mathbf{w}$ are model parameters, $\{\varphi_k\}$ are autotuned hyperparameters, $\{\lambda_k\}$ are autotuned multipliers, and $\mu$ is the single externally tuned parameter. Similarly to $\mathbf{w}$, we define $\boldsymbol{\varphi} := (\varphi_1, \dots, \varphi_K)$ and $\boldsymbol{\lambda} = (\lambda_1, \dots, \lambda_K)$.
\[ \widetilde{M}^{\res}(\theta, \mathbf{w}, \boldsymbol{\varphi},\boldsymbol{\lambda}, \mu) = \sum_{k = 1}^K \frac{\sum_{i:y_i = k} -\log f_k(x_i;\theta, \mathbf{w})}{n_k} + \lambda_k \left(  \frac{ \sum_{i:y_i \neq k} -\log(1-f_k(x_i;\theta, \mathbf{w})}{n_{\neq k}} - \varphi_k\right) + \mu \varphi_k, \] where $n_k := |\{i:y_i = k\}|, n_{\neq k} := |\{i: y_i \neq k\}|$.

The minimax problem we propose is \begin{equation}\label{appx:minimax}
    \min_{\theta, \mathbf{w}, \boldsymbol{\varphi}} \max_{\boldsymbol{\lambda}: \forall k, \lambda_k \ge 0} \widetilde{M}^{\res}(\theta, \mathbf{w}, \boldsymbol{\varphi},\boldsymbol{\lambda}, \mu), 
\end{equation} which is optimised via SGDA in \S\ref{sec:exp}.

\noindent \textbf{Overall Scheme and Model Selection} is presented in Algorithm \ref{alg:sc}. The subroutine involving the minimax solution requires training data, but this is not mentioned in the same since the focus is on model selection. The training procedure is described in \S\ref{sec:exp}. $\widehat{\mathbb{P}}_V$ refers to the empirical law on the validation dataset.

\begin{algorithm}[H]

        %\hline
        \caption{OSP-Based Selective Classifier: Model Selection}\label{alg:sc}
       %\hline
        \begin{algorithmic}[1]
            \State \textbf{Inputs}: Validation data $\{V\},$ List of $\mu$ values $\mathbf{M},$ List of $t$ values $\mathbf{T},$ Target Error $\varepsilon$.
            \For{each $\mu \in \mathbf{M},$}
                \State $(\theta(\mu), \mathbf{w}(\mu)) \gets$ minimax solution of the program (\ref{appx:minimax}) with this value of $\mu$.
            \EndFor
            \For{each $(\mu, t) \in \mathbf{M} \times \mathbf{T},$}
                \State $\s_k(\mu, t) \gets \{x: k = \argmax_{j} f_j(x; \theta(\mu), \mathbf{w}(\mu))\} \cap \{x : f_k(x; \theta(\mu), \mathbf{w}(\mu)) > t\}$.
                \State $\widehat{E}_V(\mu,t) \gets \widehat{\P}_V( \err_{\{\s_k(\mu, t)\}}).$
                \State $\widehat{C}_V(\mu,t) \gets \sum_k \widehat{\P}_V( X \in \s_k(\mu, t)).$ 
            \EndFor
            \State $(\mu_*, t_*) = \argmax_{\mathbf{M} \times \mathbf{T}} \widehat{C}_V(\mu,t) \,\,\mathrm{s.t.}\,\, \widehat{E}_V(\mu,t) \le \varepsilon.$ 
            \State \Return $\{\s_k(\mu_*,t_*)\}$.
        \end{algorithmic}%%\vspace{-0.5\topsep}
    \end{algorithm}

\newpage
\section{Experimental Details}\label{appx:expts}

The table below presents the values of the various hyperparameters used for the entries of Table \ref{table:multi_class_results}. 

\begin{table}[ht]
\tabcolsep=0.05cm
    \centering 
     \begin{tabular}{@{}c | c | c@{}} 
     \hline
     Dataset & Algorithm & Hyper-parameters \\ 
     \hline\hline
     \multirow{4}{*}{CIFAR-10}      & Softmax Response & $t=0.0445$ \\ 
                      & Selective Net  & $\lambda=32, c=0.51, t=0.24$\\ 
                      & Deep Gamblers  & $o=1.179, t=0.03$\\ 
                      & OSP-Based & $\mu=0.49, t=0.8884$ \\ 
     \hline
     \multirow{4}{*}{SVHN-10} & Softmax Response & $t=0.0224$ \\ 
                      & Selective Net  & $\lambda=32, c=0.79, t=0.86$\\ 
                      & Deep Gamblers  & $o=1.13, t=0.23$\\ 
                      & OSP-Based &  $\mu=1.67, t=0.9762$ \\ 
     \hline
     \multirow{4}{*}{Cats v/s Dogs}  & Softmax Response &  $t=0.029$ \\ 
                      & Selective Net  & $\lambda=32, c=0.7, t=0.73$\\ 
                      & Deep Gamblers  & $o=1.34, t=0.06$\\ 
                      & OSP-Based &  $\mu=1.67, t=0.9532$ \\ 
     \hline
    \end{tabular}
    \caption{Final hyper-parameters used for all the algorithms (at the desired $0.5\%$ error level) in Table \ref{table:multi_class_results}.}
    \label{tab:hyper-parameter-cross-entropy}
\end{table}

The following two tables update the numbers for Deep Gamblers to the case where we scan for 40 values of $\smallo$ in the set $[1,10)$ (as intended in the specifications) instead of $[1,2)$.

\begin{table*}[htb]
\tabcolsep=0.2cm
    \centering {%\resizebox{2\columnwidth}{!}{%
     \begin{tabular}{ c | c |  c c  c c  c c  c c } 
     \multirow{2}{*}{Dataset} & {Target} & \multicolumn{2}{c}{OSP-based}{}   & \multicolumn{2}{c}{SR} & \multicolumn{2}{c}{SN}  & \multicolumn{2}{c}{DG} \\ 
      & Error & Cov. & Error & Cov. & Error  & Cov. & Error & Cov. & Error \\ 
     \hline\hline
     %\multirow{3}{*}{CIFAR-10}     & 10\% & 100 & 9.74  & 98.52 & 10.58 & 100 & 11.07 & 97.65 & 9.97 \\
    \multirow{3}{*}{CIFAR-10}       & 2\%  & \textbf{80.6} & 1.91  & 75.1 & 2.09 & 73.0 & 2.31 & 72.9 & 1.99 \\
                    & 1\%  & \textbf{74.0} & 1.02  & 67.2 & 1.09 & 64.5 & 1.02  & 63.5 & 1.01\\
                    & 0.5\%  & \textbf{64.1} & 0.51  &  59.3 & 0.53 & 57.6 & 0.48 & 56.1 & 0.51 \\ %(1.2307692307692308, 0.04750475047504751)
     \hline
    % \multirow{3}{*}{SVHN-10}     & 10\%  & 100 & 4.15 & 99.97 & 3.86 & 100  & 4.59 & 100 & 4.38 \\
    \multirow{3}{*}{SVHN-10}      & 2\%  & \textbf{95.8} & 1.99 & 95.7 & 2.06 & 93.5  & 2.03 & 94.7 & 2.01\\
                    & 1\%  & \textbf{90.1} & 1.03  & 88.4 & 0.99 & 86.5 & 1.04 & 89.7 & 0.99\\
                    & 0.5\%  & \textbf{82.4} & 0.51  & 77.3 & 0.51 & 79.2 & 0.51 & 81.4 & 0.51 \\ %(1.4615384615384617, 0.07710771077107711)
    \hline
     %\multirow{3}{*}{Cats \& Dogs}  & 2\%  & \textbf{91.5} & 1.96  & 88.2 & 2.03 & 84.3 & 1.94 & 87.4 & 1.94 \\ 
     %              & 1\%  & \textbf{85.1} & 0.98  & 80.2 & 0.97 & 78.0 & 0.98 &  81.7 & 0.98\\ 
     %              & 0.5\% & \textbf{79.7} & 0.49  & 73.2 & 0.49 & 70.5 & 0.46 & 74.5 &  0.48\\
     \multirow{3}{*}{Cats \& Dogs}  & 2\%  & \textbf{90.5} & 1.98  & 88.2 & 2.03 & 84.3 & 1.94 & 87.4 & 1.94 \\ 
                   & 1\%  & \textbf{85.4} & 0.98  & 80.2 & 0.97 & 78.0 & 0.98 &  81.7 & 0.98\\ 
                   & 0.5\% & \textbf{78.7} & 0.49  & 73.2 & 0.49 & 70.5 & 0.46 & 74.5 &  0.48\\
    \end{tabular}}%\vspace{-10pt}
    \caption{Performance at Low Target Error. This repeats Table \ref{table:multi_class_results}, except that the hyperparameter scan for the DG method is corrected, and the entries in the DG columns are updated to show the resulting values. Notice that the performance in the last column is worse than in Table \ref{table:multi_class_results}.}
    \label{table:fixed_multi_class_results_new_table}
\end{table*} 

\newpage

\begin{table*}[hbt]
\tabcolsep=0.2cm
    \centering {%\resizebox{2\columnwidth}{!}{%
     \begin{tabular}{ c | c |  c c  c c  c c  c c } 
     \multirow{2}{*}{Dataset} & {Target} & \multicolumn{2}{c}{OSP-based}{}   & \multicolumn{2}{c}{SR} & \multicolumn{2}{c}{SN}  & \multicolumn{2}{c}{DG} \\ 
      & Coverage & Cov. & Error & Cov. & Error  & Cov. & Error & Cov. & Error \\ 
     \hline\hline
     \multirow{3}{*}{CIFAR-10}     
        & 100\%  & 100 & { 9.74 } & 99.99 & \textbf{9.58} & 100 & 11.07 & 100 & 10.95 \\ % (2.153846153846154, 0.916891689168917)
%        & 97.5\% & 97.53 & 8.29 & 97.49 & 9.96 & 97.65 & 9.96 & 97.54 & 9.47 \\
%        & 97\%   & 97.11 & 7.91 & 96.9 & 9.65 & 96.75 & 9.51 & 97.03 & 9.19 \\
        & 95\%   & 95.12 & \textbf{6.98} & 95.24 & 8.74 & 94.71 & 8.34 & 95.01 & 8.29 \\ %(1.9230769230769231, 0.5399539953995399)
        & 90\%   & 90.02 & \textbf{4.67} & 90.51 & 6.52 & 89.56 & 6.45 & 90.01 & 6.28 \\ %(2.153846153846154, 0.22762276227622763)
     \hline
     \multirow{3}{*}{SVHN-10}     
        & 100\% & 100 & 4.27 & 99.97 & \textbf{3.86} & 100 & 4.27 & 100 & 4.01  \\ % (4.230769230769231, 0.5900590059005901)
%        & 97.5\% & 97.61 & 3.02 & 97.53 & 2.68 & 97.67 & 3.31 & 97.51 & 2.85 \\
%        & 97\% & 97.03 & 2.61 & 97.08 & 2.56 & 97.13 & 3.14 & 97.02 & 2.64 \\
        & 95\% & 95.05 & \textbf{1.83} & 95.06 & 1.86 & 95.14 & 2.53 & 95.01 & 2.07 \\ %(1.4615384615384617, 0.546954695469547)
        & 90\% & 90.09 & \textbf{1.01} & 89.99 & 1.04 & 90.14 & 1.31 & 90.01 & 1.06 \\  % (1.4615384615384617, 0.24762476247624762)
     \hline
     \multirow{3}{*}{Cats \& Dogs}     
        & 100\% & 100 & 5.93 & 100 & \textbf{5.72} & 100 & 7.36 & 100 & 6.16  \\
%        & 97.5\% &  &  & 97.52 & 4.38 & 97.96 & 6.3 & 97.72 & 5.36 \\
%        & 97\% &  &  & 97.01 & 4.24 & 97.38 & 6.14 & 97.26 & 5.2 \\
        & 95\% & 95.13 & \textbf{ 2.97 } & 95.02 & 3.46 & 95.21  & 5.1 & 95.1 & 4.28  \\
        & 90\% & 90.01 & \textbf{ 1.74 } & 90.02 & 2.28 & 90.18 & 3.3 & 90.02 & 2.5  \\ %
    \end{tabular}}
    \caption{Performance at High Target Coverage. Similarly to the previous table, this repeats Table \ref{table:multi_class_results_target_cov} but with the scan for the DG method corrected. Again note the reduced performance in the final column relative to Table \ref{table:multi_class_results_target_cov}.}
    \label{table:fixed_multi_class_results_target_cov_new_table}
\end{table*} 

\newpage
\section{A deeper look at Conformal Prediction in the context of Selective Classification}\label{appx:confidence_set_rant}

This section is aimed at developing a deeper exposition of conformal prediction in general, and the confidence set methods that originate within it. We will describe these methods in some detail, discuss  their applicability to the selective classification problem, and finally discuss in detail how our proposed methods differ from such methods. 

\subsection{Why are we discussing all this in such detail?}

This paper has been submitted to two conferences (AISTATS '21, and NeurIPS '20). In both review cycles, we received comments that asserted (with no justification) that the methods we proposed are not novel, and everything we discussed is already available in the conformal prediction literature, particularly under the confidence set formulation. As we have outlined above, this is not the case. Nevertheless, since this comment was repeated, we believe that this could be a common misconception. This discussion is thus an attempt to clarify the situation, and to concretely distinguish our work, and the selective classification literature at large, from the generic conformal prediction and confidence set framework and the work done within it. Note that at the end of the day, our intention is not to dismiss this work, or to call it 'wrong' in any way - these works study interesting and legitimate problems and methodologies. Rather, we are interested in clarifying the issues of the selective classification problem, and how our work contributes to methodology regarding it. 

Our argument for distinction relies on four points. These are discussed at length in \S\ref{appx:our_method_is_diff}.

\begin{itemize}
    \item Selective Classification is not well expressed as a generic confidence set problem.
    \item While the literature does contain a binary version of a confidence set problem that is equivalent to selective classification, the natural generalisation of this to multiclass selective classification has not been pursued.
    \item While the binary version does result in a decomposition to one-sided problems, these are very distinct from the one-sided problems we pursue.
    \item Further, these one-sided problems are not exploited enough in the literature - the approach taken to these involves plugging in a nonparametric estimate of the regression function. In contrast, we exploit the decomposition to give an effective method for training selective classifiers in a discriminative setting.
\end{itemize}

\subsection{Conformal Prediction and Confidence Sets in general}

Conformal prediction is a generic learning paradigm, which for a set of output values $\mathcal{Y} = [1:K]$ involves learning functions that map features to any subset of $\mathcal{Y}$ - more succinctly, a conformal predictor is a map $h: \mathcal{X} \to 2^{[1:K]}$, where $2^\mathcal{S}$ denotes the power set of $\mathcal{S}$. Equivalently, we can interpret a conformal predictor $h$ as a collection of sets $\{\mathcal{C}_k^h\}_{k \in [1:K]},$ where $\mathcal{C}_k^h = \{x: k \in h(x) \}$. Note that thus far all we've described is the structure of conformal predictors, or equivalently, set-valued predictors.\footnote{The name `conformal prediction' arises from the work of Vovk and collaborators  \cite[see, e.g.,][]{vovk2005algorithmic, shafer2008tutorial}, which uses a measure of conformity of a prediction to a dataset. Roughly, they propose the following algorithm - let $\mathscr{D} = \{(x_i,y_i)\}_{i \in [1:n]}$ be a dataset of labelled examples. Suppose we are to predict the label of a new point $x$. The algorithm uses a measure $A: (\mathcal{X} \times \mathcal{Y})^n \times (\mathcal{X} \times \mathcal{Y}) \to \mathbb{R}$ of how well a labelled point $(x,k)$ `conforms to' or looks similar to a dataset $\mathscr{D}$. The exact proposal is illustrated by the following meta-algorithm from the survey by Shafer and Vovk \cite{shafer2008tutorial}: let $\mathscr{D}_i^k$ be obtained by dropping the $i$th example in $\mathscr{D}$, and then including the example $(x,k)$ into $\mathscr{D}$. The algorithm first computes $\pi_k := \frac{1}{n}\sum_{i = 1}^n \mathbf{1}\{\mathscr{D}_i^k \ge \alpha_k\}$, and then the prediction $h(x) = \{y: \pi_k > t\}$. The meta-algorithm is parametrised by the function $A$, and the levels $(\alpha_k)_{k \in [1:K]}$ and $t$, and the main issue becomes how to set these. Note that many variations of this meta-algorithm with the same general flavour have been proposed.} This is extremely generic, and any learning problem can be written as a conformal prediction problem, with some constraints on the form of the predictor's output (see below).

\paragraph*{The Confidence Set Formulation of Conformal Prediction} The most common approach to conformal prediction is via the confidence set formulation. This demands learning a conformal predictor $h$ such that $\mathbb{P}(Y \not\in h(X)) \le \alpha,$ where $\alpha$ is a confidence level.\footnote{In an unfortunate collision of terminology, the quantity $1-\alpha$ is also commonly called coverage in the confidence set literature. Since we use coverage for a different quantity in selective classification, we will instead call $\alpha$ the confidence level.}

\paragraph*{Standard Learning as  a Confidence Set Problem} Notice that the confidence set formulation is a strict generalisation of standard learning - the latter amounts to learning a conformal predictor such that each of the resulting $\mathcal{C}_k^h$ are disjoint, and that $\bigcup_k \mathcal{C}_k^h = \mathcal{X}$. The former condition can be succinctly written as $\forall x: |h(x)| = 1$. With these restrictions, the confidence level corresponds exactly to error rate - indeed, for any conformal predictor $h$ such that $|h| = 1,$ we can define the standard classifier $g(x) = \sum k \mathbf{1}\{h(x) = \{k\} \}$, and then $\mathbb{P}(Y \not\in h(X) ) = \mathbb{P}( g(X) \neq Y).$ Thus constraining the confidence level corresponds exactly to constraining the error rate (although nothing is known of when this is achievable at a given level).

\paragraph*{Confidence Set Methods}

The most well studied meta-method for learning confidence set predictors is a plug-in approach - these propose first learning some estimate $\widehat{\mathbb{P}}(y|x)$ of the regression function $\mathbb{P}(y|x)$, and then constructing the sets $\mathcal{C}_k = \{ \widehat{\P}(k|x) > t_k\},$ where $t_k$ are thresholds to be chosen by the algorithm. Note that in general this may not yield a unique solution - one can trade off the $t_k$ between the $k$s to generate a variety of conformal predictions with the same confidence level. Such methods have been pursued in many recent papers \cite[e.g.,][]{Lei_class_w_confidence, LeiWasserman, denis-hebiri, denis_hebiri_multi, denis_hebiri_chzhen_minimax}.

These methods typically specify two things - first, an objective that a conformal predictor should attempt to optimise, whilst maintaining a confidence level (as a constraint), and second a way to pick the $t_k$s to fit a confidence level. Variations in both can be pursued - for the latter, the meta-method described by Shafer \& Vovk \cite{shafer2008tutorial} gives a way upon using $\widehat{\P}$ to express conformity, but further methods have been proposed. The works we cited above all use a procedure called split conformal prediction, which effectively uses a validation dataset culled from the training set to set these levels (this technique originates in the work of Lei, Wasserman, and collaborators). The work by Denis, Hebiri, and collaborators observe that in certain settings (discussed below), this procedure can also utilise semi-supervised data. 

\paragraph*{Two Objectives for Confidence set methods that are relevant to selective classification}$ $\\
\noindent  \emph{Classification with confidence} was pursued in the binary case by Lei \cite{Lei_class_w_confidence}, and further by Denis \&Hebiri \cite{denis-hebiri}. This consists of learning a conformal predictor $h: \mathcal{X} \to \{ \varnothing, \{1\}, \{2\}, \{1,2\}\}$ such that for every $x,|h(x)| \ge 1$. Equivalently, the sets $\mathcal{C}_1^h$ and $\mathcal{C}_2^h$ topologically cover the input space, which is a fancy way of saying that $\mathcal{C}_1^h \cup \mathcal{C}_2^h = \mathcal{X}$. The problem is then posed as learning a predictor that optimises the following program\footnote{the work \cite{denis-hebiri} controls $\P(Y \not\in h(X) | |h(X)| = 1)$. This corresponds to a conditional objective that we do not pursue in this work, and hence will not discuss further.  }\[ \min \P(|h(X)| = 2) \textrm{ s.t. } \forall x: |h(x)| \ge 1; \mathbb{P}(Y \not \in h(X)) \le \alpha.\]
    This formulation is particularly relevant because in the binary case it is equivalent to selective classification (see the following subsection). In the main text (\S\ref{sec:post_form_comparison}) we specified a natural $K$-ary extension of this problem, which minimises $\P(|h| > 1)$ subject to the confidence level, and to the constraint that $|h| \ge 1$ everywhere. While this is equivalent to selective classification, as we discussed in the main text, to our knowledge, this exact generalisation has not appeared in the literature. Instead, the following generalisation has been pursued.
    
\noindent \emph{Least Ambiguous Set-Valued Classifiers} (LASV classifiers) were proposed by Sadinle et al. \cite{LeiWasserman}, and in a dual form by Denis \& Hebiri \cite{denis_hebiri_multi}, and further studied by Chzhen et al. \cite{denis_hebiri_chzhen_minimax}. These generalise the above problem to multiclass settings as the program \[ \min \mathbb{E}[|h(X)|] \textrm{ s.t. } \P(Y \not\in h(X) ) \le \alpha.\] In the multiclass setting, we note that the LASV classification problem is not easily related to selective classification - one way to see this is that the classifiers favour outputting an empty set to a singleton, since this lowers the objective further. This leads to a structural difference in the solutions between selective classification and the LASV classification - in the latter, the sets $\mathcal{C}_k$ are biased to be as small as possible (in probability mass) subject to capturing a large amount of the mass of each class conditional law $\P(\cdot|Y = k)$ in $\mathcal{C}_k$ \footnote{Indeed, the confidence level constraint can be expressed as $\sum \P(Y = k) \P(X \in \mathcal{C}_k| Y = k) \ge 1-\alpha$, while the objective is to minimise $\sum \P(\mathcal{C}_k).$}. While Sadinle et al. \cite{LeiWasserman} discuss the version of the problem which further constrains $h$ such that $|h| \ge 1$ everywhere, notice that the objective remains mismatched, and the solution they present is ad hoc (they propose simply outputting the label given by a standard classifier when $|h| = \varnothing$). It is also worth noting that all of these papers use plug-in estimates of the regression function to fit levels.

Importantly, even in the binary case, the approach taken by Lei \cite{Lei_class_w_confidence} is similar in spirit to the approach for the LASV problem -  they first fit a $\widehat{\P}(y|x) \approx \P(y|x)$, and then determine sets $\mathcal{D}_1 = \{x: \widehat{\P}(1|x) \ge t_1\}, \mathcal{D}_2 = \{ \widehat{\P}(2|x) \ge t_2\}$ where $t_1$ and $t_2$ are the largest possible values such that $\mathcal{C}_1$ and $\mathcal{D}_2$ together already satisfy the confidence level constraint. To satisfy the condition that $|h| \ge 1,$ the final solution is presented as  $\mathcal{C}_1 = \mathcal{D}_1, \mathcal{C}_2 = \mathcal{D}_2 \cup \mathcal{D}_1^c$.\footnote{Of course, in the binary case, for reasonable values of the confidence level, $\mathcal{D}_1$ and $\mathcal{D}_2$ should already intersect, but the point here is that the approach in similar in spirit.}

\subsection{Selective Classification as a Conformal Prediction Problem}   Selective classification corresponds to relaxing the confidence sets so that the sets $\mathcal{C}_k^h$ need not cover the space, but they remain disjoint. Indeed, this is the structure of solutions that we adopt in the formulation (\ref{SC}). The condition can be equivalently stated as $\forall x: |h(x)| \le 1$. In this case $|h(x)| = 0$ is interpreted as the rejection event. In this notation, our formulation (\ref{SC}) can be equivalently stated as \[ \max \mathbb{P}(|h(X)| = 1) \textrm{ s.t. } \forall x :|h(x)| \le 1;  \mathbb{P}(|h(X)| = 1, Y \not\in h(X)) \le \varepsilon. \]

\paragraph*{Selective Classification is not well expressed as a generic confidence set problem}

We observe that for a selective classifier (i.e., a conformal predictor with $|h| \le 1,$), the confidence set criterion \[ \mathbb{P}(Y \not \in h(X)) =  \mathbb{P}(|h(x)| = 0) + \mathbb{P}( |h(X)| = 1, Y \not\in h(X))\] adds together the probability of error and of abstention. Consequently, demanding $P(Y \not \in h(X)) \le \alpha$ no longer controls just the error rate, but also the abstention rate. As such, this means that while selective classifiers can be represented naturally as conformal predictors, the task of learning them is not easily expressed as confidence set problems. This induces a non-trivial problem if one wishes to use a generic confidence set solver to solve selective classification - the specified value of the constraint is no longer just error control. Nevertheless, it is possible to state selective classification using confidence sets \emph{if a different form of the solutions is enforced}, as we will discuss below.

\paragraph*{Turning a generic conformal predictor into a selective classifier} We make a simple observation that an arbitrary conformal predictor can yield a selective classifier by replacing non-singleton outputs with an empty set. More formall, if $h$ is a conformal predictor, then  \[ g_h(x) = \begin{cases} h(x) & |h(x)| = 1 \\ \varnothing & |h(x)| > 1 \end{cases}\] is a selective classifier. However, as above, this does not allow the confidence set objective to express the selective classification constraint as in the above case.

\paragraph*{Using Classification with Confidence to express Selective Classification as a confidence set problem} In contrast, the classification with confidence framework does allow one to express the selective classification problem as a confidence set problem. This fundamentally relies on the constraint that $|h(x)| \ge 1$ everywhere, since this implies that \[ \P(Y \not\in h(X)) = \P(|h| = 0) + \P(Y \not\in h(X), |h(X)| \ge 1) = \P(Y \not\in h(X) , |h(X)| \ge 1).\] We discuss in the main text (\S\ref{sec:post_form_comparison}) how the resulting program is equivalent to our formulation of selective classification. The transformation we propose there corresponds precisely to the above $g_h$. Notice that when measuring against the objective $\P( |h| > 1)$, we get a transformation back from selective classification to classification with confidence: for a selective classifier $g$, output the classifier with confidence \[h_g(x) = \begin{cases} g(x) & |g(x)| = 1\\ [1:K] & g(x) = \varnothing \end{cases}. \] The same is, of course, not true when measuring against the LASV classification objective.

\subsection{Our Selective Classification Methods are Distinct from Previous Conformal Prediction Approaches}\label{appx:our_method_is_diff}

More formally, our approach to selective classification does not amount to simply taking a conformal predictor learnt using existing methods, and then applying the generic transformation to selective classifiers discussed previously. This point is made via four observations, 
\begin{itemize}
    \item \emph{Selective Classification is not well expressed as a generic confidence set problem} as discussed above. This immediately removes the bulk of the conformal prediction/confidence set prediction literature from consideration.
    \item \emph{Our binary approach is distinct from classification with confidence}. As previously discussed, the approach taken in the classification with confidence setup \cite{Lei_class_w_confidence} is to find sets $\mathcal{C}_1, \mathcal{C}_2$ that are as small as possible while covering large mass of the corresponding class conditional law. More formally (upon relaxing topological coverage), the underlying one-sided problem pursued in this work (and by the extensions to LASV classification) is to learn sets $\mathcal{C}_k$ such that solve \[ \min \P(X \in \mathcal{C}_k)\,\textrm{ s.t. } \P( Y = k, X \in \mathcal{C}_k^c) \le \alpha_k. \] In sharp contrast, our formulation sets up the one-sided tasks as learning sets $\mathcal{S}_1, \mathcal{S}_2$ that are as large as possible without obtaining too much mass in the complementary class conditional law, i.e. \[ \max \P(X \in \s_k) \textrm{ s.t. } \P(X \in \s_k, Y \neq k) \le \varepsilon_k.\] 
    
    The distinction between these programs can be clarified by treating $Y \neq k$ as a null hypothesis (since it is more prevalent than $Y = k$ in typical cases). With this viewpoint, the classification with confidence and LASV solutions pursue learning a small-mass decision region $\mathcal{C}_k$ that have low false negative rate, while our OSP method pursues learning a large-mass region $\s_k$ that has low false positive rate. We note that the task of learning a small region with low false negative rate bears strong connections to statistical formulations of anomaly detection \cite{chandola2009anomaly}.
    
    \item \emph{Our multiclass approach is distinct from LASV classification} Indeed, LASV classification does not express selective classification well, and further, a difference in approach as in the above case is evident in this multiclass setting too - the LASV solutions involve solving $K$ anomaly detection type problems like above, which are distinct from our one-sided problems.
    \item \emph{We exploit our one-sided decomposition to provide effective discriminative methods rather than just plug-in schema} None of the cited confidence set methods yield a practical way of training in order to determine $\widehat{\P}(y|x)$ - their theoretical analyses all depend on standard nonparametric analyses of learning the regression function, and their proposed methods are to learn some estimate of this, and then tune the levels $\alpha_k, t$ according to a split conformal prediction approach. In sharp contrast, we exploit our formulation and the consequent relaxation and decomposition to one-sided problems to provide a novel scheme for discriminative learning that can be applied to effective train large modern function classes such as neural networks. This represents a major methodological distinction.
\end{itemize}

To summarise, our work directly addresses selective classification, has solutions with differing approaches to classification with confidence proposals, and provides practical methodologies that go beyond plug-in estimation. A final point of distinction is that our study is parametric, concentrating on learning good classifiers compared to a given class of bounded complexity, whilst the conformal prediction methods discussed tend to concentrate on nonparametric settings.

\end{document}